\documentclass[10pt,journal,compsoc]{IEEEtran}
\usepackage{color}

\usepackage{ifpdf}

\ifCLASSOPTIONcompsoc
  \usepackage[nocompress]{cite}
\else
  \usepackage{cite}
\fi

\usepackage{color}
\usepackage{array}
\usepackage{booktabs}
\usepackage{multirow}
\usepackage{amsmath}
\usepackage{amssymb}
\usepackage{graphicx}
\usepackage{amsmath,amssymb} 
\usepackage{color}
\usepackage{array}
\usepackage{multirow}
\usepackage{graphicx}
\usepackage{color}
\usepackage{times}
\usepackage{epsfig}  
\usepackage{tabularx}
\usepackage{array}
\usepackage[T1]{fontenc}
\usepackage{color}
\usepackage{booktabs}
\usepackage{amstext}
\usepackage{relsize}
\usepackage{lipsum}
\usepackage{array}
\usepackage{subfig}
\usepackage{multirow}
\usepackage{float}
\usepackage{bm}
\usepackage{bbding}
\usepackage{hyperref}
\usepackage{float}
\usepackage{amsmath,amsthm}
\usepackage{enumitem}
\usepackage[hang,flushmargin]{footmisc}

\newtheorem{myDef}{Definition}

\newtheorem{myEx}{Example}
\newtheorem{myLem}{Lemma}

\usepackage[linesnumbered,ruled,vlined]{algorithm2e}
 \SetKwInOut{KwResult}{return}

\begin{document}

\title{CHGNN: A  Semi-Supervised Contrastive Hypergraph Learning Network}

\author{Yumeng Song,
        Yu Gu,
        Tianyi Li,~\IEEEmembership{Member,~IEEE,}
        Jianzhong Qi,~\IEEEmembership{Member,~IEEE,}
        Zhenghao Liu,
        Christian S. Jensen,~\IEEEmembership{Fellow,~IEEE}
        and~Ge Yu,~\IEEEmembership{Member,~IEEE}

\thanks{Y. Song,  Y. Gu, Z. Liu, and G. Yu are with the School of Computer Science and Engineering, Northeastern University, Shenyang, Liaoning 110819, China.
E-mail: ymsong94@163.com, \{ guyu, liuzhenghao, yuge\}@mail.neu.edu.cn.}
\thanks{T. Li and C.S. Jensen are with the Department of Computer Science, Aalborg University, 9220 Aalborg, Denmark. E-mail: \{tianyi, csj\}@cs.aau.dk.}
\thanks{J. Qi is with the School of Computing and
Information Systems, The University of Melbourne, Parkville, VIC 3010, Australia. E-mail: jianzhong.qi@unimelb.edu.au. }
\thanks{Corresponding author: Yu Gu.}
}

\markboth{Journal of \LaTeX\ Class Files,~Vol.~14, No.~8, August~2015}%
{Shell \MakeLowercase{\textit{et al.}}: Bare Demo of IEEEtran.cls for Computer Society Journals}

\IEEEtitleabstractindextext{%
\begin{abstract}
Hypergraphs can model higher-order relationships among data objects that are found in applications such as social networks and bioinformatics. However, recent studies on hypergraph learning that extend graph convolutional networks to hypergraphs cannot learn  effectively from features of unlabeled data. To such learning, we propose a contrastive hypergraph neural network, CHGNN, that exploits self-supervised contrastive learning techniques to learn from labeled and unlabeled data. First, CHGNN includes an adaptive hypergraph view generator that adopts an auto-augmentation strategy and learns a perturbed probability distribution of minimal sufficient views. Second, CHGNN encompasses  an improved hypergraph encoder that considers hyperedge homogeneity to fuse information  effectively. Third, CHGNN is equipped with a joint loss function that combines a similarity loss for the view generator, a node classification loss, and a hyperedge  homogeneity loss to inject supervision signals. It also includes basic and cross-validation contrastive losses, associated with an enhanced contrastive loss training process. Experimental results on nine real datasets offer  insight into the effectiveness of CHGNN, showing that 
it outperforms 19 competitors in terms of classification accuracy consistently. 


\end{abstract}

}

\maketitle

\IEEEdisplaynontitleabstractindextext

\IEEEpeerreviewmaketitle

\IEEEraisesectionheading{\section{Introduction}\label{sec:introduction}}
\IEEEPARstart{H}{ypergraphs} generalize  graphs by introducing \textit{hyperedges} to represent higher-order relationships  among groups of nodes.  This representation is highly relevant in many real-world applications, e.g., ones that capture research collaborations. By modeling higher-order relationships,  hypergraphs can provide better performance than graph-based models in many data mining tasks~\cite{feng2019hypergraph,yadati2018hypergcn,new1,new2}. Recently, hypergraphs have been applied in diverse domains, including social networks~\cite{han2022dh}, knowledge graphs~\cite{xia2022hypergraph}, e-commerce~\cite{yuan2022community}, and computer vision~\cite{liao2021hypergraph}. 
These applications illustrate the versatility and potential of hypergraphs as a tool for capturing complex relationships among objects.

{Like graph problems (e.g., node classification and link prediction), where machine learning, especially deep learning, has yielded state-of-the-art solutions, many hypergraph problems~\cite{liao2021hypergraph,sun2021heterogeneous} are being solved by neural network-based models. In this paper, we study a fundamental problem in using neural networks for hypergraphs, which is to learn node representations for downstream tasks. The goal is to preserve the structure and attribute features of nodes in their learned embeddings so that prediction tasks based on these embeddings can be solved with higher accuracy.}

\begin{figure}[t]
\vspace{-7mm}
  \centering  
  \subfloat[HyperGNNs]{\includegraphics[width=3.2cm]{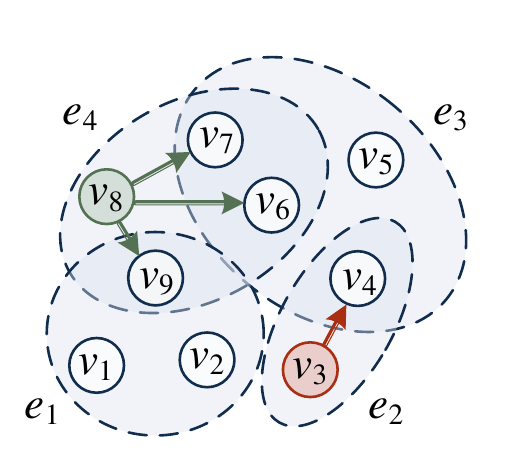}
  \label{HyperGNNs}}
  \quad
  \subfloat[CHGNN]{
    \includegraphics[width=3.2cm]{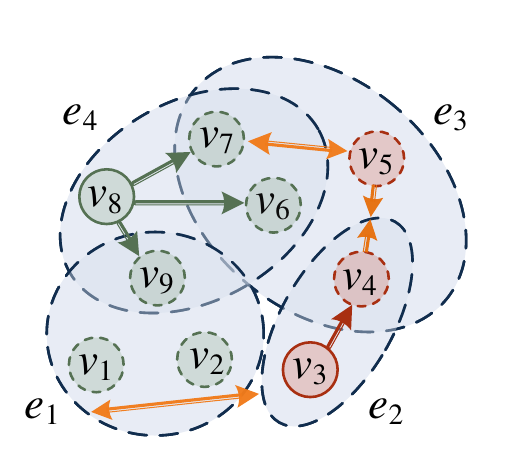}
    \label{CHGNN}} 
  \caption{Comparison of HyperGNN and CHGNN learning, where (i) colored, \textit{solid} circles denote nodes with class labels; 
  (ii) colored, \emph{dashed} circles denote nodes with model-inferred classes; 
  (iii)  uncolored circles denote nodes with model-inferred classes but that are not involved in the loss calculation; 
  (iv) dashed regions denote hyperedges; (v) red and green arrows denote class label information propagation; and (vi) orange arrows denote contrastive information propagation.}
  \label{fignew}
  \vspace{-7mm}
 \end{figure}

Since hypergraphs generalize graphs and graph learning studies~\cite{kipf2016semi,liu2023unsupervised, gao2022clusterea} have demonstrated strong results, a natural approach is to convert a hypergraph into a standard graph and apply \emph{graph neural networks} (GNNs) to learn node representations. 
However, this approach leads to poor outcomes, to be shown in the experiments (see Section~\ref{section:5}). This is because converting a hypergraph into a standard graph requires decomposing higher-order relations into pairwise relations, which causes a substantial information loss. To avoid this issue, recent studies instead extend GNNs to hypergraphs, obtaining so-called \emph{hypergraph neural networks} (HyperGNNs). Such models first aggregate node features and then generate node representations with those of their incident hyperedges~\cite{huang2021unignn,feng2019hypergraph,yadati2018hypergcn}. The aggregation process relies on  hypergraph convolution, which propagates training signals obtained from node labels.

\begin{myEx}
 Fig.~\ref{fignew} (a) shows an example of  HyperGNN learning involving nine nodes $v_i\,(1\leq i \leq 9)$. In a semi-supervised setting, where all nodes are only partially labeled, nodes without labels, i.e., $v_1$, $v_2$, $v_4$, $v_5$, $v_6$, $v_7$, and $v_9$, do not participate in the loss calculation and contribute very little to the learning process. The training loss is calculated based only on the labeled nodes, i.e.,  $v_3$ and $v_8$. Consequently, the representation learning outcome is sub-optimal, to be studied experimentally~(see Section~\ref{section:5}).
\end{myEx}

Motivated by the recent success of \emph{contrastive learning} (CL)~\cite{chen2020simple}, a self-learning technique, we propose to combine contrastive and supervised learning, in order to exploit both labeled and unlabeled data better. However, this task is non-trivial and presents two main challenges.

\noindent \emph{Challenge I: How to generate a good set of hypergraph views for CL? } 
In CL, two views are typically created by augmenting the original input data and  the model learns from both views. However, since the views are derived from the same input, there is a risk that they are very similar, which will reduce performance. Thus, it is crucial to  design the view generation so that the views offer diverse and informative perspectives on the input data.
  Existing graph and hypergraph CL data augmentation strategies \cite{velickovic2019deep,you2020graph,lee2022m}  that drop  edges randomly to form views suffer from  information loss when being applied to hypergraphs. This is because a hyperedge often represents a complex relationship among multiple nodes, meaning that randomly dropping a hyperedge makes it difficult for affected nodes to retain their full relationship features. In addition, these views usually
   contain classification-independent information to varying degrees,
  which distracts from the key information during embedding learning.

\noindent  \emph{Challenge II: How to retain  the homogeneity of hyperedges when  encoding hypergraph views for CL?}
   The homogeneity of a hyperedge captures the relatedness of the nodes that form the hyperedge.  Retaining the homogeneity of hyperedges 
 in encodings is crucial to maintain the dependencies among the nodes, which can benefit downstream tasks such as node classification. Thus, the distances between the embeddings of nodes in a hyperedge with high homogeneity should be small due to the high structural similarity. However, existing hypergraph CL models~\cite{xia2021self,zhang2021double} disregard this property, negatively impacting the quality of their learned embeddings.

To tackle the above challenges, we propose a \underline{c}onstrastive \underline{h}yper\underline{g}raph \underline{n}eural \underline{n}etwork (CHGNN). First, we propose an adaptive hypergraph view generator that utilizes the \emph{InfoMin principle}~\cite{tian2020makes} to generate a good set of views. These views share minimal information necessary for downstream tasks. Second, we introduce a method to encode nodes and hyperedges based on the homogeneity of hyperedges. Hyperedges with strong homogeneity are given higher weights during the aggregation step of nodes. Third, we design joint contrastive losses that include basic and cross-validation contrastive losses to (i) reduce the embedding distance between nodes in the same hyperedge and nodes of the same class, and (ii) verify the distribution embeddings using correlations between nodes and hyperedges, nodes and clusters, and hyperedges and clusters. Finally, to obtain more class-tailored representations, we adaptively adjust the contrastive repulsion on negative samples during training. 
\vspace{-1mm}
\begin{myEx}
\label{ex:chgnn}
Fig.~\ref{fignew}~(b) provides an example of the CHGNN learning process, where we introduce more supervision signals by inferring the classes of unlabeled nodes. For example, we encourage (i)  two nodes of the same class, e.g., $v_4$ and $v_5$, to learn embeddings with small distances in the embedding space; (ii)  nodes of different classes, e.g., $v_5$ and $v_7$, to learn embeddings with large distances; and (iii) different hyperedges $e_i\, (1 \leq i \leq 4)$ to learn embeddings with large distances. 
\end{myEx}
\vspace{-1mm}

Example{~\ref{ex:chgnn}} suggests the advantages of CHGNNs over HyperGNNs. We summarize the contributions as follows.

\begin{itemize}[leftmargin=3mm]

\item  We propose  CHGNN, a novel model for hyperedge representation learning that is based on contrastive and semi-supervised learning. The unique combination of these two strategies enables CHGNNs to exploit both labeled and unlabeled data during learning, enabling better accuracy in prediction tasks. 

\item  To enable contrastive hypergraph learning, we propose an adaptive hypergraph view generator that learns hyperedge-level augmentation policies in a data-driven manner, which provides sufficient variances for contrastive learning.  Moreover, we present a hyperedge homogeneity-aware HyperGNN (H-HyperGNN)  that takes into account the homogeneity of hyperedges when performing encoding.
    
\item  We propose a joint loss function that combines a similarity loss for view generators and supervision losses with two contrastive loss terms: basic and cross-validation contrastive losses for H-HyperGNN. They enable CHGNN to learn from both labeled and unlabeled data. 

\item We design an enhanced contrastive loss training strategy that enables the adaptive adjustment of the temperature parameter of negative samples in the loss function. The strategy gradually decreases the similarity between instances of different classes during training. 
  
\item   We evaluate  CHGNN on \textbf{nine} real datasets, finding that CHGNN consistently outperforms \textbf{nineteen} state-of-the-art proposals, including semi-supervised learning and contrastive learning methods, in terms of classification accuracy. The superior performance across all datasets demonstrates the robustness of CHGNN. 
\end{itemize}

We proceed to review related work in Section~\ref{sec:related work} and to introduce preliminaries in Section~\ref{sec:preliminaries}. Section~\ref{section:4} presents CHGNN, and Section~\ref{section:5} covers the experimental study.  Section~\ref{sec:conclusion} concludes and offers research directions.

\vspace{-3mm}
\section{Related Work}\label{sec:related work}

\subsection{Hypergraph Neural Networks} Hypergraph neural networks represent a promising approach to analyzing hypergraph-structured data. 
The basic idea of hypergraph neural networks is that nodes in the same hyperedge are often similar and hence are likely to share the same label~\cite{zhang2017re}.

HGNN~\cite{feng2019hypergraph},  the first proposal of a  HyperGNN method, performs convolution with a hypergraph Laplacian that is further approximated by truncated Chebyshev polynomials. HyperConv~\cite{bai2021hypergraph} proposes propagation between nodes that exploits higher-order relationships and local clustering structures in these.
HyperGCN~\cite{yadati2018hypergcn} further uses a generalized hypergraph Laplacian with mediators to approximate the input hypergraph.
MPNN-R~\cite{yadati2020neural} treats hyperedges as a new node type and connects a ``hyperedge node'' with all nodes that form the hyperedge. This converts a hypergraph into a graph, to which GNNs apply directly. 
AllSet~\cite{chien2021you} integrates Deep Set and Set Transformer with HyperGNNs for learning two multiset functions that compose the hypergraph neural network layers. 
HyperSAGE~\cite{arya2020hypersage} exploits the structure of hypergraphs by aggregating messages in a two-stage process, which avoids the conversion of hypergraphs into graphs and the resulting information loss.
HNHN~\cite{dong2020hnhn} extends a hypergraph to a star expansion by using two different weight matrices for node- and hyperedge-side message aggregation. 
UniGNN~\cite{huang2021unignn} provides a unified framework that generalizes the message-passing processes for both GNNs and HyperGNNs. 
Studies of HyperGNNs also exist that target tasks other than classification, e.g., hyperedge prediction~\cite{zhang2019hyper}, and that target other types of hypergraphs, e.g., recursive hypergraphs~\cite{yadati2020neural, gao2022hgnn}, which are distinct from CHGNN.

Despite the promising results obtained by the existing methods, their message-passing is limited to nodes within only a few hops. This becomes a challenge when the input graph is sparsely labeled, as the methods may fail to obtain sufficient supervision signals for many unlabeled nodes, making them unable to learn meaningful representations for such nodes. Although heterogeneous GNNs~\cite{wang2019heterogeneous, xue2021multiplex} may appear to be applicable to hypergraph tasks, by converting hypergraphs into bipartite graphs, they are not suited for such tasks.  They focus primarily on extracting semantic information from different meta-paths, which is not relevant for hypergraphs that have only one meta-path. Furthermore, experiments~\cite{chien2021you} show that (i) heterogeneous GNNs perform significantly worse than HyperGNNs; and (ii) heterogeneous GNNs do not scale well even to moderately sized hypergraphs.

While graph edges connect two nodes, edges in hypergraphs can connect multiple nodes, thereby enabling hypergraphs to model higher-order relationships. CHGNN is designed specifically for hypergraphs and may exhibit reduced performance if applied to graphs, as it cannot effectively exploit the pairwise relationships captured by graph edges. In line with this, all experiments in the paper are conducted on hypergraph datasets. (Due to the limited availability of publicly hypergraph datasets, some of the hypergraph datasets we used are constructed from graph datasets following~\cite{yadati2018hypergcn}.  The construction process introduces loss of graph information, as discussed in Section~\ref{datasets}.) As hypergraphs and graphs differ fundamentally in what they aim to model, studies of baseline methods on graph versus hypergraph datasets yield markedly different performance findings.   {\color{black} This difference can be also observed in~\cite{lee2022m}.}

\vspace{-3mm}
\subsection{Contrastive Learning} 

\subsubsection{Non-Graph Contrastive Learning}
\textit{Contrastive learning} (CL)~\cite{chen2020simple,Moco} encompasses self-supervised methods that learn representations by minimizing the distances between embeddings of similar samples and by maximizing the distances between those of dissimilar samples.
{\color{black}SimCLR~\cite{chen2020simple} introduces a novel data augmentation strategy and emphasizes the importance of diverse positive pairs. MoCo~\cite{Moco} features a momentum encoder and an efficient dictionary queue for negative samples. SimCLRv2~\cite{simclr2} and Suave~\cite{Suave} integrate contrastive learning into semi-supervised learning. 
SimCLRv2 learns embeddings through unsupervised pretraining, supervised fine-tuning, and distillation learning. Suave, the state-of-the-art semi-supervised contrastive learning model, uses class prototypes instead of cluster centroids and optimizes both supervised and self-supervised losses.

CHGNN differs from these four methods in three respects. First, CHGNN encodes and trains using both the structures and features of hypergraphs, while the four methods focus only on the features of instances. Second, the augmentation strategies of the four methods (e.g., cropping, resizing, and rotation) do not extend to hypergraphs, whereas CHGNN's strategy is tailored for hyperedges. Third,  CHGNN introduces (i) a basic contrastive loss at the  cluster and hyperedge levels and (ii) a cross-validation loss involving node-cluster, node-hyperedge, and hyperedge-cluster contrastive losses, rather than using a single contrastive loss.}

\vspace{-1mm}
\subsubsection{Graph Contrastive Learning}

Graph contrastive learning (GCL) applies  CL to GNNs.    
GCL models exist that utilize a node-level contrastive loss between nodes and the entire graph~\cite{velickovic2019deep}, high-level hidden representations~\cite{peng2020graph},  or subgraphs~\cite{jiao2020sub} without an augmentation.
Other GCL models generate contrastive samples, known as graph views, through data augmentation techniques, e.g., edge perturbation~\cite{you2020graph,zhu2021graph,li2022graph, HomoGCL}, subgraph sampling~\cite{you2020graph,hassani2020contrastive}, feature masking~\cite{you2020graph,zhu2021graph,li2022graph, HomoGCL}, learnable view generators~\cite{yin2022autogcl}, model parameter perturbation~\cite{xia2022simgrace}, and neural architecture perturbation~\cite{MAGCL}.

{\color{black}HomoGCL~\cite{HomoGCL} and MA-GCL~\cite{MAGCL} are state-of-the-art GCL models. HomoGCL improves the contrasting phase in GCL. It distinguishes intra-class neighbors from inter-class neighbors and assigns more weight to positive neighbors using a Gaussian mixture model based on graph homophily. CHGNN differs from HomoGCL in two main respects. First, CHGNN, a semi-supervised method, identifies positive nodes through training labels, while HomoGCL operates in a self-supervised manner without labels. Second, in HomoGCL, the homophily of hypergraph is the strength of dependency between multiple nodes in a hyperedge, not between pairs of nodes in a graph. In contrast, CHGNN considers the aggregation weights of hyperedges according to homophily. 
MA-GCL manipulates view encoder architectures instead of perturbing graphs for dual view encoding. It and CHGNN differ in two respects. First, CHGNN focuses on generating better views by altering the hypergraph structure and preserving its original information, whereas MA-GCL emphasizes the augmentation phase of GCL.  Second, CHGNN introduces noise through a view generator and retains key information via node importance, while MA-GCL mitigates high-frequency noise.}

\vspace{-2mm}
\subsubsection{Hypergraph Contrastive Learning}
The application of contrastive learning to hypergraphs is still evolving, with TriCL~\cite{lee2022m} being the state-of-the-art and representative model. TriCL introduces a tri-directional contrast mechanism that combines contrasting node, group, and membership labels to learn node embeddings. TriCL outperforms not only unsupervised baselines but also models trained with labeled supervision.
Xia et al.~\cite{xia2021self} improve session-based recommendations by maximizing mutual information
between representations learned. Zhang et al.~\cite{zhang2021double} address data sparsity in group recommendations by using the contrast between coarse-grained and fine-grained hypergraphs. 
HCCF~\cite{xia2022hypergraph} learns user representations by aligning explicit user-item interactions with implicit hypergraph-based dependencies for collaborative filtering. These models~\cite{xia2021self,zhang2021double,xia2022hypergraph} (i) only learn embeddings on heterogeneous social hypergraphs for non-classification tasks; and (ii) they capture differences between individual nodes rather than between nodes of different classes as a whole, as they target recommendation. As a result, they are fundamentally different from  CHGNNs.

\begin{figure*}[t]
\centering
\includegraphics[width=\textwidth]{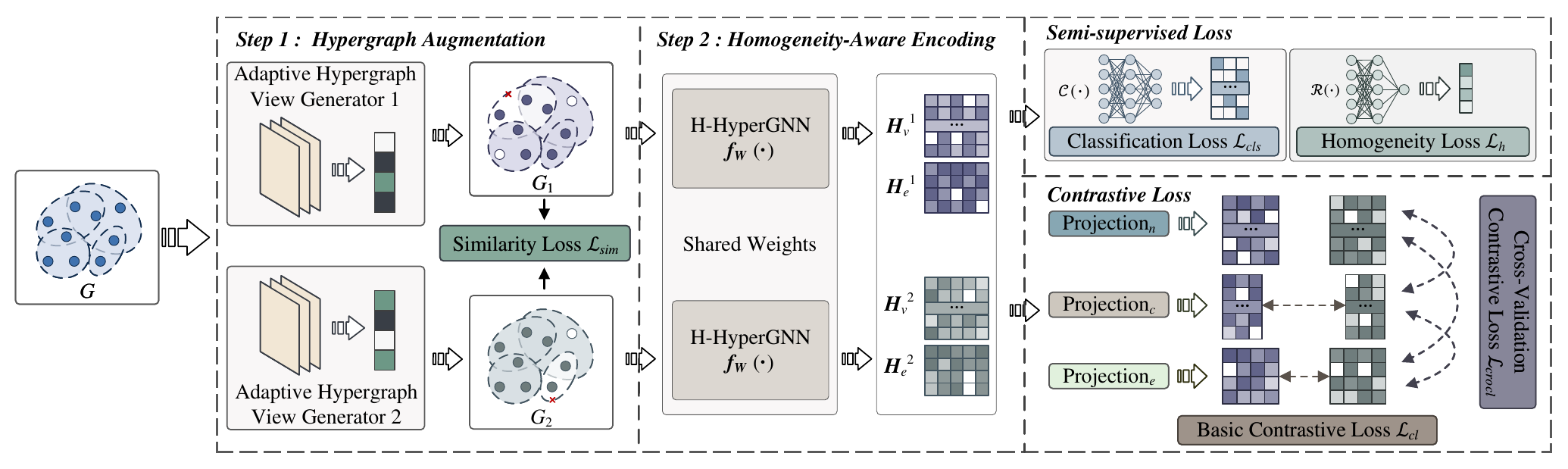}
\vspace{-3mm}
\caption{\color{black}CHGNN model overview (best viewed in color), where $\mathcal{C}(\cdot)$ is the classifier and $\mathcal{R}(\cdot)$ is the regressor. The input is a hypergraph $G$. The output is node embedding matrices $H_v^i$ and hyperedge embedding matrices $H_e^i$ of augmented hypergraphs $G_i\,(i=1,2)$. Step~1 augments the hypergraph using adaptive hypergraph view generators. Step 2 iteratively embeds views with H-HyperGNNs and trains the model using the two proposed semi-supervised and two contrastive losses.}
\label{framework}
\vspace{-5mm}
\end{figure*}
\vspace{-2mm}
\section{Preliminaries}\label{sec:preliminaries}

\begin{myDef}
A \textbf{hypergraph} is denoted as $G=(V,E)$, where $V=\{v_1,v_2, \ldots, v_n\}$ is a node set and $E\subseteq 2^V$ is a hyperedge set. Each hyperedge $e\in E$ is a non-empty subset of $V$.
\end{myDef}

\begin{myDef}
Given a hypergraph $G=(V,E)$, the \textbf{egonet} $E_{\{v_i\}}$ of a node $v_i\in V$ is the set of hyperedges that contain $v_i$.
\end{myDef}

Clearly, the higher $E_{\{v_i\}}$, the higher the connectivity of $v_i$ in the hypergraph. 
Fig.~\ref{fignew} (a) shows a hypergraph $G=(V,E)$, where $V=\{v_1, v_2, v_3, v_4, v_5, v_6, v_7,v_8,v_9\}$, $E=\{e_1, e_2, e_3, e_4\}$. $e_1 = \{v_1, v_2, v_9\}$, $e_2=\{v_3, v_4\}$, $e_3=\{v_4,v_5,v_6,v_7\}$, and $e_4=\{v_6,v_7,v_8,v_9\}$. Further,  $E_{\{v_9\}}=\{e_1, e_4\}$ due to $v_9 \in e_1 \cap  e_4$.

\begin{myDef}
The \textbf{feature matrix} of a hypergraph $G=(V,E)$ is denoted as $\bm{X} \in  \mathbb{R} ^{|V|\times F}$, where each row $\bm{x}_i$ of  $\bm{X}$ is an $F$-dimensional feature vector of a  node $v_i$.   
\end{myDef}

\begin{myDef}
The \textbf{class label vector} of a hypergraph $G=(V,E)$ is denoted as $\bm{y} \in \mathbb{N}^{|V|}$, where each element $y_i$ of $\bm{y}$  represents the class label of  $v_i$ and $\mathbb{N}$ is 
a set of natural numbers. 

\end{myDef}

\begin{myDef}
 The \textbf{homogeneity}~\cite{lee2021hyperedges} of a hyperedge $e\in E$, denoted  $homo(e)$ is given by:
 \vspace{-1mm}

\begin{equation}
	homo(e) := \begin{cases}
		\sigma (\frac{\sum _{\{v_i,v_j\} \in \binom{e}{2}} |E_{\{v_i\}} \cap E_{\{v_j\}}|}{|\binom{e}{2}| } )  &\text{if } |e|>1\\
    1  &\text{otherwise},
		   \end{cases}
  \label{eq1}
  \vspace{-1mm}
\end{equation}
where $\binom{e}{2}$ is the set of node pairs in $e$, $|\binom{e}{2}|$ is the cardinality of $\binom{e}{2}$, $|E_{\{v_i\}} \cap E_{\{v_j\}}|$ is the degree of pair $\{v_i,v_j\}$, and $\sigma(\cdot)$ is the sigmoid function.

\end{myDef} 

Homogeneity measures the strength of the dependencies between the nodes in a hyperedge. A high homogeneity of a hyperedge $e$ means that pairs of nodes in $e$ appear frequently in the edge set $E$. This indicates a strong dependency and a high likelihood of similar attributes (e.g., labels) among all node pairs.

\begin{myDef}
 Give a hypergraph $G = (V, E)$,  the feature matrix $\bm{X}$ of $G$ and the labeled  set $V_L$ ($|V_L| \ll |V|$) with the label set $Y_L =\{y_i | v_i \in V_L\}$, \underline{c}ontrastive \underline{h}yper\underline{g}raph \underline{n}eural \underline{n}etwork (CHGNN)  performs transductive graph-based semi-supervised learning for \textbf{semi-supervised node classifiction}, where the labeled  set $Y_{un} =\{y_j | v_j \in V-V_L\}$ is identified.
\end{myDef}

\vspace{-5mm}
\section{CHGNN}
\label{section:4}

\subsection{Model Overview}

CHGNN follows the common two-branch framework of GCL models~\cite{you2020graph} and encompasses two steps: hypergraph augmentation and homogeneity-aware encoding, as shown in Fig.~\ref{framework}.

\noindent \textit{\textbf{Hypergraph augemtation.}} The input hypergraph $G = (V, E)$ is augmented to obtain $G_1 =(V_1, E_1)$ and $G_2 = (V_2, E_2)$ (cf. Fig.~\ref{framework}) using two adaptive hypergraph view generators.   
We detail the hypergraph augmentation in Section~\ref{section:4.2}.

\noindent \textit{\textbf{Homogeneity-aware encoding.}} The two augmented hypergraph views, $G_1$ and $G_2$, are fed separately into two hyperedge homogeneity-aware HyperGNNs {(H-HyperGNNs)}  to encode the nodes and hyperedges into latent spaces, 
described in Section~\ref{section:4.3}.

\noindent \textit{\textbf{Loss function.}}  We propose a loss function that encompasses two semi-supervised and two contrastive loss terms.
The semi-supervised loss utilizes label and hyperedge homogeneity information to generate supervision signals that propagate to unlabeled nodes during training. 
The basic contrastive loss forces embeddings of the same cluster (and hyperedge) computed from $G_1$ and $G_2$ to be similar, while embeddings of different clusters (and hyperedges) are expected to be dissimilar.  
The cross-validation contrastive loss maximizes the embedding agreement between the nodes, the clusters, and the hyperedges. 
We detail the loss function in Section~\ref{section:4.4}.

\noindent \textit{\textbf{Enhance training.}}
We apply increased repulsion to the negative samples that are difficult to distinguish by adjusting the temperature parameters so that nodes can obtain more class-tailored representations with limited negative samples.
 The enhanced training process is covered in Section~\ref{section:4.5}.

\vspace{-3mm}
\subsection{Hypergraph Augmentation}
\label{section:4.2}
Hypergraph augmentation is  essential for creating hypergraph views that preserve the topological structure of the hypergraph and maintain the node labels.
 We propose augmentation schemes that maintain significant topological and semantic information of hypergraphs for classification, thereby improving CL performance.

\vspace{-2mm}
\subsubsection{Adaptive Hypergraph View Generator}
Given a hypergraph $G$, we generate two hypergraph views $G_1$ and $G_2$ by perturbing $G$ with three operations: preserve hyperedges (P), remove hyperedges (R), and mask some nodes in the hyperedges (M). Note that graph augmentation typically includes removing nodes and masking node features in addition to the above operations. However, perturbation of nodes does not apply to hypergraphs due to their hyperedges that represent higher-order relationships rather than the pairwise relationships. 
Directly perturbing a node interferes with the aggregation of $d_v~\times~d_e$ neighbor nodes, while perturbing a hyperedge only affects $d_e$,  where $d_v$ and $d_e$ is the average node and hyperedge degree in a hypergraph, respectively. To preserve the original topological and semantic information and retain greater variance, we employ finer-grained hyperedge perturbation instead of node perturbation.

We use the HyperConv layers~\cite{bai2021hypergraph} to obtain the hyperedge embedding matrix {$\bm{M}_E$} by iteratively aggregating nodes and hyperedges. Each hyperedge embedding represents the probabilities of preserving, removing, and masking operations. We employ the Gumbel-softmax~\cite{jang2016categorical} to sample from these probabilities, and then the output {$\bm{v}_{aug}$} assigns an augmentation operation  to each hyperedge. We mask a node $v_i$ based on its \emph{importance} during the masking operation. 

\vspace{-3mm}
\subsubsection{Overlappness-based Masking}

To maintain the semantics (i.e., class)  of node labels during hypergraph augmentation, two masking 
 principles should be followed. First, a node is considered more important if its egonet contains more distinct nodes. 
  Second, if the hyperedges contained in the egonet of a node overlap less with the hyperedges contained in the egonets of other nodes, the node is considered more important.
 We provide an example to illustrate the two principles.

\vspace{-1mm}
\begin{myEx}
In Fig.~\ref{fignew} (a),   $E_{\{v_9\}}=\{e_1, e_4\}$ contains 6 nodes: $v_1$, $v_2$, $v_6$, $v_7$, $v_8$, $v_9$, while $E_{\{v_4\}}=\{e_2, e_3\}$ contains 5  nodes: $v_j$ ($3 \leq j \leq 7$). In this case, masking node $v_9$ results in higher information loss than masking node $v_4$. Second, the number of nodes contained in $E_{\{v_6\}}=\{e_3, e_4\}$ is 6, which is the same as for $E_{\{v_9\}}$. However, $v_6$ is less important than $v_9$ because (i) $E_{\{v_6\}}=E_{\{v_{7}\}}$, meaning that $e_3$ and $e_4$ can still be connected via  $v_{7}$ even if $v_6$ is masked; and (ii) $E_{\{v_9\}}\cap E_{\{v_1\}}=\{e_1\}\neq E_{\{v_9\}}$, and $E_{\{v_9\}}\cap E_{\{v_8\}}=\{e_4\}\neq E_{\{v_9\}}$,  meaning that $e_1$ and $e_4$ cannot be connected if $v_9$ is masked.
\end{myEx}

\vspace{-1mm}
We introduce the  \textit{overlappness}~\cite{lee2021hyperedges} of an egonet to evaluate the significance of a node $v_i$ in line with the above two principles.

\begin{myDef}
The \textbf{overlappness} of the egonet of $v_i$, denoted by $o (v_i)$, is defined as:
\vspace{-2mm}
\begin{equation}
  o(v_i):= \frac{\sum_{e\in E_{\{v_i\}}}|e|}{|\bigcup _{e\in E_{\{v_i\}}}e|}
\label{eq3} 
\end{equation}
\vspace{-5mm}
\end{myDef}

\noindent When the number of distinct nodes in the egonet of $v_i$ is large, or the hyperedges in the egonet of $v_i$ overlap less, the value of $o (v_i)$ is small, i.e., the node is more important (cf. Formula~\ref{eq3}). To normalize  overlapness values, which may vary considerably, we use $w_{v_i} =\log o (v_i)$. Next, we apply it to compute the  probability of masking $v_i$.
\vspace{-3mm}
\begin{equation}
  p_{v_i}= \min \{ \frac{w_{nmax} - w_{v_i}}{w_{nmax}-w _{navg} } \cdot p_{node}, p_\tau \},
\label{eq4}
\vspace{-2mm}
\end{equation}

\noindent where $p_{node}$ is a hyperparameter that determines the overall probability of masking any node (i.e., the sparsity of the resulting views), $w_{nmax}$ and $w _{navg}$ are the maximum and  average of the normalized overlappness in $G$, and $p_{\tau }$ is a cut-off probability. 

\vspace{-2mm}
\subsubsection{View Generator Training}
A good set of views should share only the minimal information needed for high performance on a classification task. In line with this,
we train a view generator in an end-to-end manner to generate task-dependent views, allowing the classification loss to guide the generator to retain classification information.
A similarity loss is used to minimize the mutual information between the views generated by following~\cite{tian2020makes}. The similarity between the views is measured by comparing the outputs of the view generators:
\vspace{-1mm}
\begin{equation}
  \mathcal{L}_{sim}= sim(\bm{v}_{aug}^1, \bm{v}_{aug}^2),
\label{eq5}
\vspace{-1mm}
\end{equation}
where $\bm{v}_{aug}^1$ and $\bm{v}_{aug}^2$ are the augmentation vectors for generators $G_1$ and $G_2$ in Fig.~\ref{framework}, respectively, and $sim(\cdot)$ is the similarity loss. In this paper, we adopt the mean squared error loss.

\vspace{-3mm}
\subsection{Hyperedge Homogeneity-aware HyperGNN}
\label{section:4.3}

The aggregation process above treats all hyperedges containing a node $v_i$ equally and ignores their semantic differences.

\vspace{-1mm}
\begin{myEx}
  Consider a publication domain classification problem as shown in Fig.~\ref{fig3}. 
    When updating the embedding of $v_i$ for subsequent classification, existing HyperGNNs treat $e_1$, $e_2$, and $e_3$ equally  (see Fig.~\ref{fig3} (a)). However,  $e_2$ should be given  lower weight, as it contains nodes from distinct domains and has more diverse semantics (see Fig.~\ref{fig3} (b)). Specifically, when an author publishes in diverse domains, there is a higher uncertainty regarding which domain a paper   by the author belongs to.
\end{myEx}
\vspace{-1mm}

\begin{figure}
\centering
\subfloat[Equal aggregation ]{
\includegraphics[width=3cm]{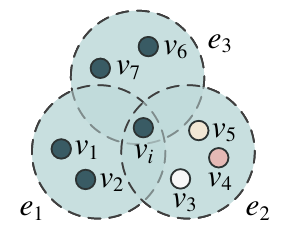}}  
\quad\quad\quad
\subfloat[Ideal aggregation]{ 
\includegraphics[width=3cm]{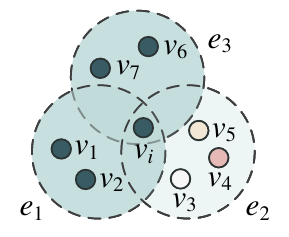}} 
\caption{\color{black}Aggregation considering nodes homogeneity. (i) Each dashed circle represents a hyperedge, and each node in a circle represents an individual paper. (ii) Nodes in the same circle, or hyperedge,  share an author. (iii) The hyperedge coloring represents distinct aggregation weights in CHGNN: darker colors indicate higher weights, i.e., $e_1$ and $e_3$ have higher aggregation weights than $e_2$ in (b). (iv) Nodes are colored to denote different domains, i.e., $v_2$, $v_3$, $v_4$, and $v_5$ belong to separate domains, unlike $v_1$, $v_2$, $v_6$, $v_7$, and $v_i$.}
\label{fig3}
\vspace{-6mm}
\end{figure}

The above example suggests that it is crucial to take into account the semantic information of hyperedges. Motivated by this, we propose an H-HyperGNN that weighs the hyperedges by their homogeneity of semantic information. We propagate the features as follows:
\begin{equation}
  \begin{split}
  \bm{h}_{e,e_i}^{(l)} & = \frac{1}{|e_i|} \sum _{v_j\in e_i} \bm{h}_{v,v_j}^{(l-1)}  \\
  \bm{h}_{v,v_i}^{(l)} & = \bm{W}(\bm{h}^{(l-1)}_{v,v_i} + \sum_{e_j\in E_{\{v_i\}}} homo(e_j)\cdot \bm{h}_{e,e_j}^{(l)}),
  \end{split}
 \label{eq6}
\end{equation}
where $\bm{W}$ is a learnable parameter matrix and $homo(e)$ denotes the homogeneity of a hyperedge $e\in E$, which is calculated following Formula~\ref{eq1}. Next,  $\bm{h}_{e,e_i}^{(l)}$ and $\bm{h}_{v,v_i}^{(l)}$  denote the embeddings of  $e_i$ and $v_i$ at the $l$-th  iteration, where $\bm{h}_{v,v_i}^{(0)} = \bm{x}_i$ and $\bm{x}_i$ is the input feature vector of $v_i$. We use a two-layer H-HyperGNN in our experiments (see Section~\ref{section:5}). The final layer output of the H-HyperGNN forms the learned embeddings of the nodes and hyperedges, which are used  for loss computation  in the next stage.

\vspace{-5mm}
\subsection{Model Training}
\label{section:4.4}

\subsubsection{Semi-Supervised Training Loss}
\label{section:4.4.1}
\noindent\textit{\textbf{Classification  loss.}} 
The final embedding of each labeled node $v_i$ is computed by taking the average of its two embedding vectors $\bm{h}_{v,v_i}^1$ and $\bm{h}_{v,v_i}^2$  obtained from $\bm{H}_{v}^1$ and $\bm{H}_{v}^2$ generated from $G_1$ and $G_2$  (see Fig.~\ref{framework}). The final embedding is then input into
a classifier network (we use a two-layer MLP for simplicity) $\mathcal{C}(\cdot)$, where the cross-entropy loss is used as the classification loss.
\vspace{-2mm}
\begin{equation}
  \mathcal{L}_{cls}=-\sum_{v_i\in V_L} y_i\ln \mathcal{C}(\frac{\bm{h}_{v,v_i}^1+\bm{h}_{v,v_i}^2}{2}),
  \label{eq8}
  \vspace{-1mm}
\end{equation}
where  $V_L$ is the set of labeled nodes and $y_i$ is the label of node $v_i$.

\noindent\textit{\textbf{Hyperedge homogeneity loss.}}
{\color{black}We denote the hyperedge embedding of view $G_i$  by $\bm{H}^i_e$,  $i\in \{1,2\}$.  We feed $\bm{H}_{e}^{i}$ into a two-layer MLP that serves as a regressor $\mathcal{R}(\cdot)$ to predict the homogeneity of all hyperedges based on their embeddings and then compute the homogeneity loss (see Fig.~\ref{framework}). } Since hyperedge embeddings are generated by node aggregation, homogeneity can be viewed as a measure of node similarity within a hyperedge. {The homogeneity loss aims to decrease the similarity between node embeddings in heterogeneous hyperedges and increase the similarity in homogeneous hyperedges. To achieve this, we compute the mean squared error between the true and predicted homogeneity values as the hyperedge homogeneity loss.}
 
\vspace{-1.5mm}
\begin{equation}
  \mathcal{L}_{h}= \lambda _{h} \cdot (\frac{1}{2} \sum_{i \in \{1,2\}}\mathit{MSE}(\mathcal{R}(\bm{H}_{e}^i), \bm{v}_{homo})),
  \label{eq9}
  \vspace{-2mm}
\end{equation}
 where $\lambda _{h}>0$ is a hyperparameter that weighs the importance of hyperedge homogeneity loss,  $\bm{H}_{e}^i$ is the hyperedge embedding matrix generated from view $G_i$, and $\bm{v}_{homo}$ is the homogeneity vector of $E$. The $j$-th element of $\bm{v}_{homo}$  is denoted as $homo(e_j)$.

\vspace{-3mm}
\subsubsection{Contrastive Loss}
\label{section:4.4.2}

We define the output matrix in CL.

\begin{myDef}
$\bm{Z}_b^a$ is an \textbf{output matrix} of a projection head, where $a\,(\in \{1,2\})$ denotes views and $b\, (\in \{e, c,n \})$ denotes the type of projection. Further $c$ represents  cluster projection, $e$ represents  hyperedge projection, and $n$ represents  node projection. $\bm{Z}_{b:i}^a$ denotes the $i$-th column of $\bm{Z}_b^a$, and $\bm{z}_{b,w}^a$ denotes the vector of $w$ in $\bm{Z}_b^a$.
\end{myDef}
\vspace{-1mm}

\noindent \textit{\textbf{Basic contrastive loss.}} Basic contrastive loss is similar to GCL~\cite{you2020graph, hassani2020contrastive, zhu2021graph, wan2021contrastive, yin2022autogcl}  that contrasts the same data from different views.
Two contrastive loss notions are used commonly: node-level contrastive loss~\cite{you2020graph, zhu2021graph,hassani2020contrastive} and cluster-level contrastive loss~\cite{li2021contrastive}.
The former aims to distribute nodes uniformly in the hyperplane, while the latter focuses on clustering nodes with similar categories and separating nodes with different categories.
Clearly, the latter aligns with the objective of classification, while the former does not. Thus, we include cluster-level contrastive loss in the basic contrastive loss but exclude node-level contrastive loss.
Inspired by InfoNCE loss~\cite{oord2018representation}, we propose a hyperedge-level contrastive loss to enhance the hyperedge embeddings.  We thus incorporate also hyperedge-level contrastive loss in the basic contrastive loss.

\noindent (i) \emph{Cluster-level contrastive loss}. 
This loss  compares the cluster membership probability distributions learned from different views for each node. To get cluster membership probability distributions, $\bm{H}_v^1$ and $\bm{H}_v^1$ are loaded into a cluster projection to obtain cluster matrices $\bm{Z}_c^1$ and $\bm{Z}_c^2$. We treat the distributions of the same cluster as positive samples.
Conversely, the cluster membership probability distributions of different nodes learned from different views are treated as negative samples,
to differentiate between node clusters rather than nodes. 
Ideally, each node cluster should correspond to nodes of the same class. To achieve this, we project node embeddings into a space with  dimensionality equal to the number of data classes (i.e., the target number of clusters). In particular, the projection is carried out using a two-layer MLP.  
The $i$-th dimension value represents the probability of a node belonging to the $i$-th class.  

Let $\bm{Z}_c^1$ and $\bm{Z}_c^2$  be the output matrix of the projection head $\mathcal{P}(\cdot)$ for the node embeddings, i.e., $\bm{Z}_c^i = \mathcal{P}(\bm{H}_v^i), i = \{1, 2\}$.
The cluster-level contrastive loss~\cite{li2021contrastive} is:
\vspace{-1mm}
\begin{equation}
  \mathcal{L}_{c}= \frac{1}{2C}\sum _{i=1}^C (\ell_{c}^{1}(\bm{Z}^1_{c:i})+\ell_{c}^{2}(\bm{Z}^2_{c:i})),
\label{eq10}
\vspace{-1mm}
\end{equation}
where $C$ is the number of clusters. Formula~\ref{eq10} aims to achieve two goals. First, for any $i \in [1, C]$, $\bm{Z}_{c:i}^1$ is very similar to  $\bm{Z}_{c:i}^2$, i.e., a node $v_j$ from the input hypergraph has a similar probability of falling into the $i$-th class, regardless of whether its embedding is learned from $G_1$ or $G_2$. Second, $\bm{Z}_{c:i}^1$ and $\bm{Z}_{c:j}^2$ should be dissimilar when $i\neq j$ because a node should have dissimilar probabilities of falling into different classes. Formula~\ref{eq10}  is an average of the cluster contrastive losses  $\ell_{c}^{1}$ and $\ell_{c}^{2}$ computed from the two views. Since the two loss terms are symmetric, we only show $\ell_{c}^1(\bm{Z}^1_{c:i})$:
\vspace{-1mm}
\begin{equation}
  \ell_{c}^1(\bm{Z}^1_{c:i})= -\log \frac{\exp(s(\bm{Z}^1_{c:i}, \bm{Z}^2_{c:i})/\tau_{c} )}{\sum_{j=1}^C \exp (s(\bm{Z}^1_{c:i}, \bm{Z}^2_{c:j})/\tau_{c})}
\label{eq11}
\end{equation}
Here, $\tau_{c}$ is a temperature parameter corresponding to $\ell_{c}^1$   that controls the softness of the contrasts among positive and negative sample pairs, and $s(\cdot)$ is a vector similarity function --- we use  inner product.

\noindent (ii) \emph{Hyperedge-level contrastive loss}. 
This loss aims  to maximize the mutual information of hyperedge representations between different views.
We project hyperedge embeddings $\bm{H}_e^1$ and $\bm{H}_e^2$ with the projection head (the two-layer MLP) and output $\bm{Z}_e^1$ and $\bm{Z}_e^2$. 
For a hyperedge $e_i$, $\bm{z}^1_{e,{e_i}}$ and $\bm{z}^2_{e,e_i}$ form a positive sample pair
, while the projected embeddings of different hyperedges from different views form the negative sample pairs.  We thus define the hyperedge-level contrastive loss as:
\vspace{-1mm}
\begin{equation}
  \mathcal{L}_{e}= \frac{1}{2|E|}\sum _{e_i\in E} (\ell_{e}^{1}(\bm{z}^1_{e,e_i})+\ell_{e}^{2}(\bm{z}^2_{e,e_i}))
\label{eq12}
\vspace{-1mm}
\end{equation}
Similar to the cluster-level loss, $\ell_{e}^{1}$ and $\ell_{e}^{2}$ in Formula~\ref{eq12} denote the hyperedge contrastive losses of the two views. We thus only show $\ell_{e}^1(\bm{z}^1_{e,e_i})$: 
\vspace{-1mm}
\begin{equation}
  \ell_{e}^1(\bm{z}^1_{e,e_i})= -\log \frac{\exp(s(\bm{z}^1_{e,e_i}, \bm{z}^2_{e,e_i})/\tau_{e} )}{\sum_{e_j\in E} \exp (s(\bm{z}^1_{e,e_i}, \bm{z}^2_{e,e_j})/\tau_{e})},
\label{eq13}
\vspace{-1mm}
\end{equation}
 where $\tau_{e}$ is a temperature parameter corresponding to $\ell_{e}^1$.

Overall, we define the basic contrastive loss as the weighted sum of cluster-level and hyperedge-level contrastive losses:
\vspace{-1mm}
\begin{equation}
  \mathcal{L}_{cl} = \lambda _{c}  \mathcal{L}_{c} + \lambda _{e} \mathcal{L}_{e},
\label{eq14}
\vspace{-1mm}
\end{equation}
\noindent  where $\lambda _{c}>0$ and $\lambda _{e}>0$ are hyperparameters for balancing the importance of  $\mathcal{L}_{c}$ and $\mathcal{L}_{e}$.

\noindent\textbf{Cross-Validation Contrastive Loss.} In GCL, contrastive loss is typically designed to compare  graph elements of the same data type across different views, leading to, e.g., node-level and edge-level contrastive loss. However, this strategy ignores topological relationships between elements of different data types. To rectify the shortcoming, TriCL~\cite{lee2022m} introduces a contrastive loss that verifies the relationships between nodes and hyperedges, thus improving the performance of GCL.  
Building on that study, we propose a cross-validation contrastive loss that matches the  embedding distributions between nodes and clusters, nodes and hyperedges, and hyperedges and clusters.  This enables the embeddings of different types of data (i.e., nodes, clusters, and hyperedges) to obtain additional information beyond that contained in the data of the same type, 
resulting in significantly improved performance~(see Section~\ref{section:5.5}).

\noindent(i) \emph{Node-cluster contrastive loss.} This loss aims to distinguish the relationship between a node's embedding and its own clustering result from the relationship between the node's embedding and the clustering results of the other nodes. The node's embedding and its own clustering result are considered as a positive sample pair, while the node's embedding and the clustering results of the other nodes are considered as a negative sample pair. 
 Continuing the framework in Fig.~\ref{framework}, we calculate the
 loss between node embedding $\bm{Z}_v^1$ of $G_1$ and the clustering result $\bm{Z}_c^2$ of $G_2$,
where $\bm{Z}_v^1$ and $\bm{Z}_c^2$ are projected from $\bm{H}_v^1$ and $\bm{H}_v^2$, respectively. 
Then $\bm{z}^1_{v,v_1}$ and $\bm{z}^2_{c,c_1}$ constitute a positive sample pair, while $\bm{z}^1_{v,v_1}$ and $\bm{z}^2_{c,c_3}$ constitute a negative sample pair. The  node-cluster contrastive loss is defined as:
\vspace{-1mm}
\begin{equation}
  \mathcal{L}_{nc}= \frac{1}{2|V|}\sum _{v_i\in V} (\ell_{nc}(\bm{z}^1_{v,v_i},\bm{z}^2_{c,v_i})+\ell_{nc}(\bm{z}^2_{v,v_i},\bm{z}^1_{c,v_i})),
\label{eq15}
\vspace{-1mm}
\end{equation}
where $\ell_{nc}$ is the loss function for each node-cluster pair:  
\begin{equation}
\ell_{nc}(\bm{z}_{v,v_i}^p,\bm{z}_{c,v_i}^q) = -\log \frac{\exp(\mathcal{D}(\bm{z}_{v,v_i}^p, \bm{z}_{c,v_i}^q)/\tau_{nc} )}{\sum_{v_j\in V} \exp (\mathcal{D}(\bm{z}_{v,v_j}^p, \bm{z}_{c,v_j}^q)/\tau_{nc})},
\label{eq16}
\end{equation}
where $p,q \in \{1,2\}$, $\mathcal{D}(\cdot)$ is a discriminator, and $\tau_{nc}$ is a temperature parameter corresponding to $\ell_{nc}$.

\noindent(ii) \emph{Node-hyperedge contrastive loss.}
Following TriCL~\cite{lee2022m}, (i) the embeddings of node $v_i$ from one view and the hyperedge $e_j\,(v_i\in e_j)$ from the other view are treated as a positive sample pair, and (ii) the embeddings of node $v_k$ and the hyperedge $e_j\, (v_k \notin e_j)$ from different views are treated as a negative sample pair. Continuing the framework in Fig.~\ref{framework}, we calculate the loss between node embedding $\bm{Z}_v^1$ of $G_1$ and hyperedge embedding $\bm{Z}_e^2$ of $G_2$,
where  $\bm{Z}_v^1$ and $\bm{Z}_e^2$ are projected from $\bm{H}_v^1$ and $\bm{H}_e^2$ respectively. 
Given
$ E=\{e_1,e_2,e_3,e_4\},E_{\{v_1\}} = \{e_1, e_2\}$,  and an anchor 
$v_1$, $\bm{z}^1_{v,v_1}$ and $\bm{z}^2_{e,e_1}$ are positive samples, $\bm{z}^1_{v,v_1}$ and $\bm{z}^2_{e,e_2}$ are positive samples, and $\bm{z}^1_{v,v_1}$ and $\bm{z}^2_{e,e_4}$ are negative samples. The node-hyperedge contrastive loss $\mathcal{L}_{ne}$ is defined as: 
\begin{equation}
  \mathcal{L}_{ne}= \frac{1}{2M}\sum _{v_i\in V}\sum _{e_j\in E_{\{v_i\}}}(\ell_{ne}(\bm{z}^1_{v,v_i},\bm{z}^2_{e,e_j})+\ell_{ne}(\bm{z}^2_{v,v_i},\bm{z}^1_{e,e_j})),
\label{eq17}
\vspace{-2mm}
\end{equation}
where $M=\sum _{v_i\in V}|E_{\{v_i\}}|$ and $\ell_{ne}$ is defined as: 
\vspace{-1mm}
\begin{equation}
\begin{aligned}
 \ell_{ne}(\bm{z}^p_{v,v_i},\bm{z}^q_{e,e_j}) = 
 & -\log \frac{\exp(\mathcal{D}(\bm{z}^p_{v,v_i}, \bm{z}^q_{e,e_j})/\tau_{ne} )}{\sum_{v'\in V} \exp (\mathcal{D}(\bm{z}^p_{v,v'}, \bm{z}^q_{e,e_j})/\tau_{ne})} \\
 & -\log \frac{\exp(\mathcal{D}(\bm{z}^p_{v,v_i}, \bm{z}^q_{e,e_j})/\tau_{ne} )}{\sum_{e'\in E} \exp (\mathcal{D}(\bm{z}^p_{v,v_i}, \bm{z}^q_{e,e'})/\tau_{ne})},
\end{aligned}
\label{eq18}
\end{equation}
where $p, q \in \{1,2\}$ and $\tau_{ne}$  is a temperature parameter corresponding to $\ell_{ne}$.

\noindent(iii) \emph{Hyperedge-cluster contrastive loss.}
This loss resembles the node-hyperedge contrastive loss in that the clustering result can be viewed as another mapping of node embeddings. If $v_i \in e_j $, the cluster results of $v_i$ and the embedding of $e_j$ are considered as a pair of positive samples; otherwise, they are designated as a pair of negative samples. 
Continuing the framework in Fig.~\ref{framework}, we calculate the 
loss between cluster results $\bm{Z}_c^1$ of $G_1$ and hyperedge embedding $\bm{Z}_e^2$ of $G_2$,
where $\bm{Z}_c^1$ and $\bm{Z}_e^2$ are projected from $\bm{H}_v^1$ and $\bm{H}_e^2$, respectively.
Given $e_1 = \{v_1,v_2\}$ and $e_2=\{v_1\}$  (i.e. $E_{\{v_1\}} = \{e_1, e_2\}$), $\bm{z}^1_{v,v_1}$ and $\bm{z}^2_{e,e_1}$ form a pair of positive samples, while $\bm{z}^1_{v,v_1}$ and $\bm{z}^2_{e,e_4}$ form a pair of negative samples.  The hyperedge-cluster contrastive loss is defined as:
\begin{equation}
   \mathcal{L}_{ec}= \frac{1}{2M}\sum_{v_i\in V}
   \sum _{e_j \in E_{\{v_i\}}}
   (\ell_{ec}(\bm{z}^1_{c,v_i},\bm{z}^2_{e,e_j})+\ell_{ec}(\bm{z}^2_{c,v_i},\bm{z}^1_{e,e_j})),
\label{eq19}
\end{equation}
where $M=\sum _{v_i\in V}|E_{\{v_i\}}|$. The objective function $\ell_{ec}$ for hyperedge-cluster contrast is defined as: 
\vspace{-1mm}
\begin{equation}
 \begin{aligned}
 \ell_{ec}(\bm{z}^p_{c,v_i},\bm{z}^q_{e,e_j}) = 
 & -\log \frac{\exp(\mathcal{D}(\bm{z}^p_{c,v_i}, \bm{z}^q_{e,e_j})/\tau_{ec} )}{\sum_{v'\in V} \exp (\mathcal{D}(\bm{z}^p_{c,v'}, \bm{z}^q_{e,e_j})/\tau_{ec})} \\
 & -\log \frac{\exp(\mathcal{D}(\bm{z}^p_{c,v_i}, \bm{z}^q_{e,e_j})/\tau_{ec} )}{\sum_{e'\in E} \exp (\mathcal{D}(\bm{z}^p_{c,v_i}, \bm{z}^q_{e,e'})/\tau_{ec})},
\end{aligned}
\label{eq20}
\vspace{-1mm}
\end{equation}
where $p,q \in \{1,2\}$ and $\tau_{ec}$ is a temperature parameter corresponding to $\ell_{ec}$. Integrating Formulas~\ref{eq15},~\ref{eq17}, and~\ref{eq19}, the cross-validation contrastive loss becomes:
\vspace{-1mm}
\begin{equation}
\mathcal{L}_{crocl} = \lambda _{nc}  \mathcal{L}_{nc} + \lambda _{ne} \mathcal{L}_{ne} + \lambda _{ec} \mathcal{L}_{ec},
\label{eq21}
\vspace{-1mm}
\end{equation}
where $\lambda _{nc}$, $\lambda _{ne}$, and $\lambda _{ec}$ are positive hyperparameters that balance  the importance of the three losses. Finally, the overall loss function of CHGNN is:
\vspace{-1mm}
\begin{equation}
  \mathcal{L} = \mathcal{L}_{sim} + \mathcal{L}_{cls} + \mathcal{L}_{h}+\mathcal{L}_{cl} +\mathcal{L}_{crocl},
\label{eq22}
\vspace{-1mm}
\end{equation}
where $\mathcal{L}_{sim}$ is the similarity loss from hypergraph view generators, $\mathcal{L}_{cls}$ is the classification loss, $\mathcal{L}_{h}$ is the hyperedge homogeneity loss, $\mathcal{L}_{cl}$ is the basic contrastive loss, and $\mathcal{L}_{crocl}$ is the cross-validation loss. We train CHGNN in an end-to-end manner. This (i) leads to classification-oriented view generators that sample the hypergraph to satisfy the InfoMin principle~\cite{tian2020makes} and (ii) provides  better-augmented views for CL.

\vspace{-4mm}
\subsection{Enhanced Contrastive Loss Training}
\label{section:4.5}

\begin{figure}
\centering
\includegraphics[width=6.5cm]{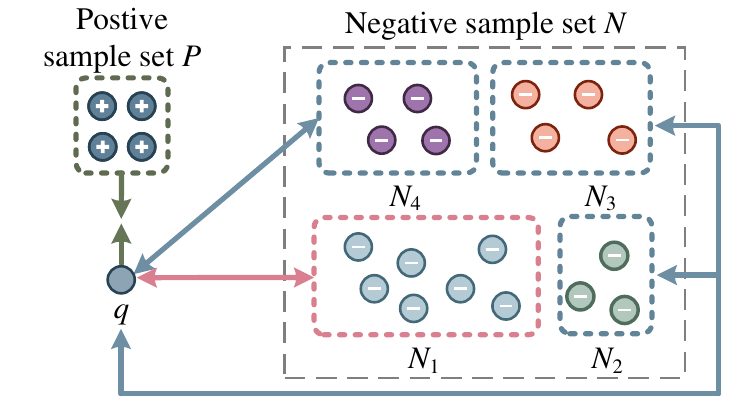}
\caption{\color{black}An example of enhanced contrastive loss training. (i) Each circle represents a sample.
(ii) Circles with the same color as $q$ are in the same class as $q$. Circles with different colors are in different classes, with colors closer to $q$ indicating a higher similarity to $q$.
(iii) Green arrows indicate attraction, while blue and red arrows indicate repulsion in CL, with red indicating stronger repulsion.}
\label{fig_4_5}
\vspace{-5mm}
\end{figure}

\vspace{-0.5mm}
\subsubsection{Adaptively Distancing Negative Samples}
\label{section:4.5.1}
CL with the InfoNCE  loss~\cite{oord2018representation} typically requires a large number of negative samples to be effective, which is a practical challenge. Given the number of both clusters and classes $C$ (see Section~\ref{section:4.4.2}), the number of negative samples for a training node is $C-1$ (see Section~\ref{section:4.4.2}). We have $C<70$ for all of the nine datasets used in experiments (see Section~\ref{section:5}), which implies a scarcity of negative samples in real-world scenarios. This hinders the efficient  learning of feature information. To address this, 
we propose to identify negative samples that are difficult to distinguish from an instance (i.e., a target node cluster or a hyperedge) by applying a novel adaptive distancing method. 
The method effectively creates a ``repulsion'' between the negative samples and an instance, improving their separability and removing the necessity of training a large number of negative samples. This way, the proposed method can enable efficient and effective feature learning, even with few negative samples. We introduce the concept of \textit{relative penalty}~\cite{wang2021understanding}, before detailing the method.
\vspace{-1mm}
\begin{myDef}
\label{def:9}
Given an instance $q$ and a set of negative samples  $N=\{n_1, n_2, \ldots, n_{|N|}\}$, $\mathcal{R}=\{r_{n_1}, r_{n_2}, \ldots, r_{n_{|N|}}\}$ is a set of \textbf{relative penalties}, where $r_{n_i}$ is used to adjust the similarity between $q$ and nodes in $n_i$.
\end{myDef}
\vspace{-1mm}
We use $s_{z}(q, n_i)$ to denote the similarity between the embeddings of an instance $q$ and a negative sample $n_i$. According to the above discussion, it is desirable to impose a high $r_{n_i}$ on negative samples with a high $s_z(q, n_i)$, in order to increase the embedding distances between nodes in $n_i$ and $q$ and thus to differentiate them. 

\vspace{-1mm}
\begin{myDef}
\label{def:10}
Given an instance $q$, a set of negative samples $N=\{n_1, n_2, \ldots, n_{|N|}\}$, and a set of relative penalties $\mathcal{R}=\{r_{n_1}, r_{n_2}, \ldots, r_{n_{|N|}}\}$, the \textbf{desired ranking} of relative penalties in $\mathcal{R}$ satisfies $\forall s_z(q, n_{i})> s_z(q, n_{j}) \; (r_{n_i} > r_{n_j}), 1 \leq i,j  \leq m$.
\vspace{-1mm}
\end{myDef}

\begin{myEx}
In Fig.~\ref{fig_4_5}, $P$ is a positive sample set, and $N$ is a negative sample set partitioned into non-overlapped subsets $N_1$, $N_2$, $N_3$, and $N_4$ ($N=N_1\cup N_2 \cup N_3 \cup N_4$). For simplicity, we assume that all negative samples in the same subset have the same embedding similarity to $q$, i.e., 
$s_z(q,n_{i})=s_z(q,n_{j})$, and we let $s_z(q,N_k)=s_z(q,n_{i})\,(n_i, n_j\in N_k, 1\leq k \leq 4)$, where $s_z(q,N_k)$ denotes the similarity between $q$ and any  $q'\in N_k$. Then, we expect $r_{n_i}=r_{n_j}\,(n_i, n_j\in N_k, 1\leq k \leq 4)$ (see Definition~\ref{def:9}). We thus let $r_{N_k}=r_{n_i}$ $(n_i\in N_k, 1\leq k \leq 4)$. Next, given $s_z(q, N_{1})> s_z(q, N_2)>s_z(q, N_3)>s_z(q, N_4)$, ideally we should have  $r_{N_1}>r_{N_2}>r_{N_3}>r_{N_4}$. 
\end{myEx}

 The relative penalty $r_{n_i}$ associated with 
$n_i$ is~\cite{wang2021understanding}:
\begin{equation}
r_{n_i} = \left\lvert \frac{\partial \ell_q}{\partial s_z(q,n_i)} \right\rvert \bigg/ \left\lvert \frac{\partial \ell_q}{\partial s_z(q,p)} \right\rvert,
\label{eq25}
\end{equation}
where $q$ is the training instance, $p$ and $n_i$ are positive and negative samples, $s_z(\cdot)$ is an embedding similarity function, $N$ is a negative sample set, and $\ell_q$ is the InfoNCE loss function of $q$ in each epoch: 
\begin{equation}
\ell_q= -\log \frac{\mathcal{F}(q,p,\tau_{p})}{\mathcal{F}(q,p,\tau_{p}) +\sum_{n_i \in N} \mathcal{F}(q,n_i,\tau_{n})},
\label{eq24}
\end{equation}
where $\mathcal{F}(a,b,c) = \exp(s_z(a,b)/c)$ and $\tau_{n}$ and $\tau_{p}$ are 
 temperature hyperparameters that are fixed and set to the same value during normal contrastive loss training. 
The objective is to dynamically adjust the set of relative penalties $\mathcal{R}$ in order to satisfy the desired ranking (cf. Definitions~\ref{def:9} and \ref{def:10}). Formulas~\ref{eq25} and~\ref{eq24} suggest that the temperature hyperparameter $\tau_{n_i}$ uniquely determines $r_{n_i}$ for the negative sample $n_i$. 
This transfers the objective to the dynamic modification of temperature parameters  for different negative samples during training.

\vspace{-1mm}
\begin{myLem}
 Given a positive sample $p$, a training instance $q$, and a negative sample set $N=\{n_1, n_2\}$ with $s_z(q, n_{1})> s_z(q, n_{2})$ then we have  $r_{n_1} > r_{n_2}$ if $\tau_{n_1}=\tau^{ub}/s_z(q,n_1)$ and $\tau_{n_2}=\tau^{ub}/s_z(q,n_2)$, where $\tau^{ub}$ is the upper bound of the temperature  parameter.
\end{myLem}

\begin{proof}

According to Formulas~\ref{eq25} and \ref{eq24}: 
\begin{equation}
\begin{aligned}
r_{n_1}= \left\lvert \frac{\partial -\log \frac{\mathcal{F}(q,p,\tau_{p})}{\mathcal{F}(q,p,\tau_{p}) +\sum_{n_i \in N} \mathcal{F}(q,n_i,\tau_{n_i})}}{\partial s_z(q,n_1)} \right\rvert
\bigg/ \left\lvert \frac{\partial \ell_q}{\partial s_z(q,p)} \right\rvert
\end{aligned}
\label{f:proof_ni}
\vspace{-1mm}
\end{equation}
Similarly, we have:
\vspace{-1mm}
\begin{equation}
r_{n_2}= \left\lvert \frac{\partial -\log \frac{\mathcal{F}(q,p,\tau_{p})}{\mathcal{F}(q,p,\tau_{p}) +\sum_{n_i \in N} \mathcal{F}(q,n_i,\tau_{n_i})}}{\partial s_z(q,n_2)} \right\rvert
\bigg/ \left\lvert \frac{\partial \ell_q}{\partial s_z(q,p)} \right\rvert
\label{f:proof_nj}
\end{equation}
Since Formulas~\ref{f:proof_ni} and~\ref{f:proof_nj} have the same denominator, proving $r_{n_1}>r_{n_2}$ is equivalent to proving:
\vspace{-1mm}
\begin{equation}
\frac{\mathcal{F}(q,n_1,\tau_{n_1})}{\mathcal{U}_1}>\frac{(\mathcal{F}(q,n_2,\tau_{n_2})}{\mathcal{U}_2},
\label{f:F}
\end{equation}
\noindent  where $\mathcal{U}_i=\tau_{n_i}(\mathcal{F}(q,n_1,\tau_{n_1})+\mathcal{F}(q,n_2,\tau_{n_2})+\mathcal{F}(q,p,\tau_{p}))$ ($i=1,2$). Given  $\tau_{n_i}=\tau^{ub}/s_z(q,n_i)$,  we have $\mathcal{F}(q,n_i,\tau_{n_i})=\exp(s_z(q,n_i)^2/\tau^{ub})$. Since $s_z(q,n_1)>s_z(q,n_2)$, $\tau_{n_1}< \tau_{n_2}$ and $\mathcal{F}(q,n_1,\tau_{n_1})>\mathcal{F}(q,n_2,\tau_{n_2})$. Thus, $r_{n_1}>r_{n_2}$ according to Formula~\ref{f:F}.
\end{proof}

When $|N|>2$,   $\tau_{n_i}=\tau^{ub}/s_z(q,n_i)$  ($n_i \in N$) guarantees that $r_{n_1}>r_{n_2}>\ldots>r_{n_{|N|}}$ if $s_z(q,{n_1})>s_z(q,{n_2})>\ldots>s_z(q,{n_{|N|}})$. This can be derived straightforwardly.
 
\vspace{-2mm}
\subsubsection{Training Strategy}
\label{section:4.5.2}

In Section~\ref{section:4.5.1}, we detail the process and benefits of adaptively distancing negative samples. However, the representations of instances (i.e., node clusters or hyperedges)  contain only poor feature information during the initial stages of training. 
Thus, we cannot start distancing two  negative samples from each other before we have learned sufficient information for classification. This requires us to monitor the progress of the learning.

\noindent\textit{\textbf{Differentiating cluster-level loss.}} If centers of  different clusters (i.e., classes) are well separated,  the cluster-level representations can capture rich feature information. Given a training set $V_L=\{V_L^1, V_L^2, \ldots, V_L^C\}$, where $C$ is the number of classes and $V_L^i$ includes the labeled nodes of class $c_i$, we calculate the embedding center $\bm{h}^{(k)}_{c_i}$ of cluster embddings $\bm{z}^{(k)}_{c,v_{\{i\}}}$ ($v_{\{i\}} \in V_L^i$) after $k$ epochs. Here, $\bm{z}^{(k)}_{c,v_{\{i\}}} = 1/2(\bm{z}^{1, (k)}_{c,v_{\{i\}}}+\bm{z}^{2,(k)}_{c,v_{\{i\}}})$ and  $(k)$ denotes the $k$-th epoch embeddings. We denote  the similarities between centers $\bm{h}^{(k)}_{c_i}$ and $\bm{h}^{(k)}_{c_j}$ at the $k$-th epoch by $m_{ij}^{(k)}$. If $\overline{m}^{(k)}_c=\frac{1}{\binom{C}{2}}\sum_{(i,j) \in \binom{C}{2}}m_{ij}^{(k)}$ is below a threshold $\varepsilon_c$, we adjust the fixed temperature parameters $\tau_{c}$ of Formula~\ref{eq11}. For each negative sample $c_i$ with anchor $c_q$ at the $k$-th epoch, the temperature parameter $\tau_{c,(q,i)}^{(k)}$ is defined as follows:
\vspace{-1mm}
\begin{equation}
\tau_{c,(q,i)}^{(k)}=\begin{cases}

\tau^{ub}_{c}  &\text{if} \quad \overline{m}_c^{(k)}\geqslant \varepsilon_{c} \\

\tau^{ub}_{c}/s_z(\bm{Z}^{1,(k)}_{c:q}, \bm{Z}^{2,(k)}_{c:i}) &\text{otherwise}
\end{cases},
\label{eq31}
\vspace{-2mm}
\end{equation}
where $\varepsilon_{c}$ is a threshold hyperparameter, $\tau_{c}^{ub}$ is the upper bound, and $\bm{Z}^{1,(k)}_{c:q}$ and $\bm{Z}^{2,(k)}_{c:i}$ are two cluster embedded representations at the $k$-th epoch.

\noindent\textit{\textbf{Differentiating hyperedge-level contrastive loss.}} This procedure is similar to that of differentiating node-level contrastive loss. 
Since no class labels for hyperedges are available, we construct hyperedges with the training set $V_L$. All nodes in $V_L^i$ are assigned to a  hyperedge $e_{V_L^i}$.  Since a hyperedge describes a relationship among multiple nodes and $V_L^i$ includes the labeled nodes of class $c_i$, we assume that the constructed hyperedge $e_{V_L^i}$  belongs to class $c_i$. If the distributions of the constructed hyperedge representations differ from each other, the representation of the hyperedge is rich in feature information. Thus,  when the mean $\overline{m}_e^{(k)}$ of the similarities between hyperedge representations is below a threshold $\varepsilon_e$ at the $k$-th epoch, we adjust the temperature parameters $\tau_{e}$ of Formula~\ref{eq13}.  For each negative sample $e_i$ with  anchor $e_q$ at the $k$-th epoch, the temperature parameter $\tau_{e,(q,i)}^{(k)}$ is defined as follows:
\vspace{-1mm}
\begin{equation}
\tau_{e,(q,i)}^{(k)}=\begin{cases}

\tau_{e}^{ub}  &\text{if} \quad \overline{m}_e^{(k)} \geqslant \varepsilon_{e} \\

\tau_{e}/s_z(\bm{Z}^1_{e,e_q}, \bm{Z}^2_{e,e_i}) &\text{otherwise}
\end{cases},
\label{eq32}
\vspace{-1mm}
\end{equation}
where $\varepsilon_{e}$ is a threshold hyperparameter, $\tau_{e}^{ub}$ is the upper bound, and $\bm{Z}^{1,(k)}_{e,e_q}$ and $\bm{Z}^{2,(k)}_{e,e_i}$ are two hyperedge embedded representations at the $k$-th epoch.

  \vspace{-5mm}
  \section{Experimental Study}
\label{section:5}

 \subsection{Experimental Setup} 

 \subsubsection{Datasets} 
 \label{datasets}
 We evaluate CHGNN using nine publicly available real datasets. Table~\ref{table1} summarizes statistics of their hypergraphs.

\begin{table}[t]
 \centering
 \caption{Hypergraphs  statistics of the nine datasets. }
 \renewcommand\tabcolsep{3pt}
 \begin{tabular}{lrrrrr}
  \toprule  
  \textbf{Dataset} & \textbf{\#nodes} & \textbf{\#hyperedges} & \textbf{\#attributes} & \textbf{\#classes} &\textbf{\color{black}label ratio}\\
  \midrule
  CA-Cora            & 2,708           & 1,072            & 1,433                 & 7          & {\color{black}5.2\%}        \\
  CA-DBLP           & 43,413           & 22,535          & 1425                   & 6    & {\color{black}4.0\%}               \\  
  CC-Pubmed           & 19,717           & 7,963           & 500                   & 3      & {\color{black}0.8\%}            \\ 
CC-Citeseer         & 3,312            & 1,079           & 3,703                  & 6         & {\color{black}4.2\%}          \\
CC-Cora             & 2,708            & 1,579            & 1,433                 & 7            & {\color{black}5.2\%}       \\
20news           & 16,242           & 100          & 100                   & 4                & {\color{black}50\%}   \\  
ModelNet           & 12,311           & 12,311          & 100                   & 40     & {\color{black}50\%}              \\  
NTU          & 2,012           & 2,012          & 100                   & 67            & {\color{black}50\%}        \\  
Mushroom           & 8,124           & 298          & 22                   & 2         & {\color{black}50\%}           \\  
\bottomrule
 \end{tabular}

 \label{table1}
 \vspace{-2mm}
 \end{table}

 \begin{table}
\renewcommand\tabcolsep{7pt}
\centering
\caption{Hyperparamter settings for CHGNN.}
\begin{tabular}{lccccccc}
\toprule

\textbf{Dataset}  
& $\bm{\lambda_e}$  
& $\bm{\lambda_{nc}}$
& $\bm{\lambda_{ne}}$ 
& $\bm{\lambda_{ec}}$
& \textbf{nhid}
& \textbf{nproj}
\\

  \midrule

CA-Cora
&0.2
&0.2
&0.2
&0.2
&64
&16\\
   
CA-DBLP   
&0.0625
& 0
& 0.25
&0.25
&64
&16\\
CC-Pubmed
& 0.9
&0.8
&0.7
&0.7
&64
&16\\
CC-Citeseer
&0.8
& 1
&0.8
& 0.8
&64
&16\\
CC-Cora
&0.7
& 0.5
&0.9
& 0.9
&64
&16\\
20news
& 0.1
&  0.8
& 0.6
&  0.6
&64
&16\\
ModelNet
& 0.7
& 0.1
&0.032
& 0
&256
&16\\
NTU
& 0.3
& 0.1
&  0.7
& 0.7
&512
&512\\
Mushroom
& 0.2
&0
&0.1
& 0.024
&64
&16\\

\bottomrule
\end{tabular}
  
  \label{table_hp}
  \vspace{-5mm}
  \end{table}

  \begin{table*}
  \centering
   \renewcommand\tabcolsep{5pt}
   
  \caption{ \color{black}Comparison results (\%), where (i) $*$ indicates unsupervised learning methods; (ii) values after $\pm$ are standard deviations; (iii) the best and second-best performances are marked in bold and underlined, respectively; and (iv) OOM denotes out of memory on a 24GB GPU.}

\begin{tabular}{@{}llcccccccccc@{}}
  \toprule

    &{\color{black}\textbf{Dataset}}	&{\color{black}	\textbf{CA-Cora}}	&{\color{black}	\textbf{CA-DBLP}}	&{\color{black}	\textbf{CC-Pubmed}}	&{\color{black}	\textbf{CC-Citeseer}}	&{\color{black}	\textbf{CC-Cora}}	&{\color{black}	\textbf{20news}}	&{\color{black}	\textbf{ModelNet}}	&{\color{black}	\textbf{NTU}}	&{\color{black}	\textbf{Mushroom}}
    
    \\ \midrule
   
   \multirow{5}{*}{\rotatebox{90}{\textbf{\color{black}Non-Graph}}}

  &{\color{black}MLP				}	
&{\color{black}	57.8	$\pm$	4	}	
&{\color{black}	80.6	$\pm$	0.8	}	
&{\color{black}	73.3	$\pm$	1.6	}	
&{\color{black}	59.1	$\pm$	4	}	
&{\color{black}	59.9	$\pm$	2.6	}	
&{\color{black}	81.1	$\pm$	1.1	}	
&{\color{black}	95.9	$\pm$	0.8	}	
&{\color{black}	87.2	$\pm$	2.8	}	
&{\color{black}	98.4	$\pm$	1.7	}	
\\					
					
&{\color{black}SimCLR$^*$				}	
&{\color{black}	44.2	$\pm$	1.6	}	
&{\color{black}	76.9	$\pm$	0.4	}	
&{\color{black}	72.5	$\pm$	1.4	}	
&{\color{black}	45.5	$\pm$	4.4	}	
&{\color{black}	39.5	$\pm$	1.3	}	
&{\color{black}	80.7	$\pm$	0.7	}	
&{\color{black}	94.4	$\pm$	3.9	}	
&{\color{black}	82.8	$\pm$	2.5	}	
&{\color{black}	96.3	$\pm$	0.3	}	
\\

&{\color{black}MoCo$^*$					}
&{\color{black}	42.2	$\pm$	2		}
&{\color{black}	75.9	$\pm$	0.5		}
&{\color{black}	66.3	$\pm$	1.5		}
&{\color{black}	43.8	$\pm$	1.4		}
&{\color{black}	38.4	$\pm$	1.5		}
&{\color{black}	80	$\pm$	0.4		}
&{\color{black}	96.5	$\pm$	2.2		}
&{\color{black}	82.1	$\pm$	2		}
&{\color{black}	88.6	$\pm$	0.8		}
\\					
					
&{\color{black}SimCLRv2					}
&{\color{black}	55	$\pm$	5.3		}
&{\color{black}	80.1	$\pm$	0.4		}
&{\color{black}	74.5	$\pm$	2.8		}
&{\color{black}	57.8	$\pm$	4		}
&{\color{black}	57.6	$\pm$	3.7		}
&{\color{black}	81.4	$\pm$	0.4		}
&{\color{black}	95.3	$\pm$	0.2		}
&{\color{black}	85.9	$\pm$	1.3		}
&{\color{black}	98.3	$\pm$	0.2		}
\\					
					
&{\color{black}Suave					}
&{\color{black}	61.3	$\pm$	2.8		}
&{\color{black}	80.8	$\pm$	2.9		}
&{\color{black}	75.2	$\pm$	2.3		}
&{\color{black}	62.8	$\pm$	2.3		}
&{\color{black}	60.8	$\pm$	3		}
&{\color{black}	81.5	$\pm$	0.6		}
&{\color{black}	96.6	$\pm$	0.3		}
&{\color{black}	89.1	$\pm$	0.7		}
&{\color{black}	98.6	$\pm$	0.8		}
\\

   \midrule
   \multirow{6}{*}{\rotatebox{90}{\textbf{\color{black}Graph}}}

&{\color{black}GCN				}
&{\color{black}	63.7	$\pm$	1.8	}
&{\color{black}	85.6	$\pm$	0.3	}
&{\color{black}	74.6	$\pm$	1	}
&{\color{black}	59.4	$\pm$	2.3	}
&{\color{black}	60.4	$\pm$	1.3	}
&{\color{black}	83.2	$\pm$	0.4	}
&{\color{black}	95.9	$\pm$	0.2	}
&{\color{black}	85.1	$\pm$	0.9	}
&{\color{black}	93.6	$\pm$	0.2	}
\\				
				
&{\color{black}GAT				}
&{\color{black}	65.7	$\pm$	1.9	}
&{\color{black}	85.2	$\pm$	0.3	}
&{\color{black}	74.8	$\pm$	1.8	}
&{\color{black}	59.3	$\pm$	1.7	}
&{\color{black}	60.2	$\pm$	1.7	}
&{\color{black}	79.7	$\pm$	7	}
&{\color{black}	95.5	$\pm$	0.2	}
&{\color{black}	82.8	$\pm$	1.4	}
&{\color{black}	90.3	$\pm$	13	}
\\				
				
&{\color{black}DGI$^*$				}
&{\color{black}	64.4	$\pm$	0.3	}
&{\color{black}	79.3	$\pm$	0.3	}
&{\color{black}	70.9	$\pm$	0.2	}
&{\color{black}	59.7	$\pm$	0.3	}
&{\color{black}	61.4	$\pm$	0.3	}
&{\color{black}	82.8	$\pm$	0.3	}
&{\color{black}	95.9	$\pm$	0.2	}
&{\color{black}	79.5	$\pm$	0.2	}
&{\color{black}	85.7	$\pm$	0.1	}
\\				
				
&{\color{black}SimGRACE$^*$				}
&{\color{black}	63.1	$\pm$	1.5	}
&{\color{black}	82.2	$\pm$	0.3	}
&{\color{black}	74.5	$\pm$	1.2	}
&{\color{black}	57.2	$\pm$	1.8	}
&{\color{black}	64.3	$\pm$	2	}
&{\color{black}	82.5	$\pm$	0.5	}
&{\color{black}	95.8	$\pm$	0.3	}
&{\color{black}	76.6	$\pm$	1.4	}
&{\color{black}	90	$\pm$	0.3	}
\\				
				
&{\color{black}HomoGCL$^*$				}
&{\color{black}	58.3	$\pm$	1.6	}
&{\color{black}	OOM			}
&{\color{black}	74.8	$\pm$	1	}
&{\color{black}	55.8	$\pm$	1.2	}
&{\color{black}	54.3	$\pm$	1.5	}
&{\color{black}	82.4	$\pm$	0.3	}
&{\color{black}	95.7	$\pm$	0.2	}
&{\color{black}	74.3	$\pm$	0.7	}
&{\color{black}	91.7	$\pm$	0.3	}
\\				
				
&{\color{black}MA-GCL$^*$				}
&{\color{black}	68.5	$\pm$	1.9	}
&{\color{black}	85.6	$\pm$	0.2	}
&{\color{black}	74.1	$\pm$	1.9	}
&{\color{black}	65.8	$\pm$	1.3	}
&{\color{black}	68.9	$\pm$	1.7	}
&{\color{black}	82.9	$\pm$	0.4	}
&{\color{black}	93.8	$\pm$	0.3	}
&{\color{black}	80.3	$\pm$	1.1	}
&{\color{black}	92.5	$\pm$	0.3	}
\\

  \midrule
  \multirow{10}{*}{\rotatebox{90}{\textbf{\color{black}Hypergraph}}}
  
  &{\color{black}HyperConv					}
	&{\color{black}	65.2	$\pm$	1.5	}
	&{\color{black}	88.3	$\pm$	0.1	}
	&{\color{black}	40.1	$\pm$	7	}
	&{\color{black}	37.5	$\pm$	3.3	}
	&{\color{black}	43.9	$\pm$	3.3	}
	&{\color{black}	81	$\pm$	0.5	}
	&{\color{black}	94.3	$\pm$	0.1	}
	&{\color{black}	86.1	$\pm$	0.8	}
	&{\color{black}	97.9	$\pm$	0.2	}
	\\				
					
	&{\color{black}HGNN				}
	&{\color{black}	69.9	$\pm$	3.3	}
	&{\color{black}	78.5	$\pm$	0.3	}
	&{\color{black}	75.3	$\pm$	1.1	}
	&{\color{black}	60.7	$\pm$	2.2	}
	&{\color{black}	62.7	$\pm$	3.7	}
	&{\color{black}	81.1	$\pm$	0.6	}
	&{\color{black}	95.3	$\pm$	0.6	}
	&{\color{black}	86.2	$\pm$	2.3	}
	&{\color{black}	96.7	$\pm$	0.9	}
	\\				
					
	&{\color{black}HyperGCN				}
	&{\color{black}	72.1	$\pm$	5	}
	&{\color{black}	88.3	$\pm$	0.6	}
	&{\color{black}	73.1	$\pm$	9.6	}
	&{\color{black}	60.1	$\pm$	3.6	}
	&{\color{black}	62.5	$\pm$	3.5	}
	&{\color{black}	81.3	$\pm$	0.5	}
	&{\color{black}	73.7	$\pm$	13.2	}
	&{\color{black}	53	$\pm$	8.7	}
	&{\color{black}	46.4	$\pm$	0.1	}
	\\

	&{\color{black}HNHN				}
	&{\color{black}	65.5	$\pm$	4.5	}
	&{\color{black}	87.6	$\pm$	0.7	}
	&{\color{black}	70.7	$\pm$	2.3	}
	&{\color{black}	55.3	$\pm$	2.4	}
	&{\color{black}	62.4	$\pm$	3	}
	&{\color{black}	80.9	$\pm$	0.5	}
	&{\color{black}	97.5	$\pm$	0.4	}
	&{\color{black}	86.5	$\pm$	5.1	}
	&{\color{black}	98	$\pm$	1.2	}
	\\

	&{\color{black}HyperSAGE				}
	&{\color{black}	72.9	$\pm$	3.1	}
	&{\color{black}	87.6	$\pm$	0.8	}
	&{\color{black}	73.1	$\pm$	1.4	}
	&{\color{black}	61.3	$\pm$	2.6	}
	&{\color{black}	67.4	$\pm$	3.8	}
	&{\color{black}	81.8	$\pm$	0.6	}
	&{\color{black}	97.5	$\pm$	0.8	}
	&{\color{black}	85.4	$\pm$	3.8	}
	&{\color{black}	97.8	$\pm$	1.9	}
	\\				
					
	&{\color{black}UniGNN				}
	&{\color{black}	74.3	$\pm$	1.9	}
	&{\color{black}\underline{89.0	$\pm$	0.3}		}
	&{\color{black}	76	$\pm$	1.4	}
	&{\color{black}	62.6	$\pm$	3.2	}
	&{\color{black}	68.3	$\pm$	2.6	}
	&{\color{black}\underline{83.3	$\pm$	1.1}		}
	&{\color{black}	97	$\pm$	0.3	}
	&{\color{black}	85.8	$\pm$	2.8	}
	&{\color{black}	98.1	$\pm$	1.1	}
	\\				
					
	&{\color{black}AST				}
	&{\color{black}	71.1	$\pm$	3.8	}
	&{\color{black}	87.7	$\pm$	0.5	}
	&{\color{black}	64.2	$\pm$	9.8	}
	&{\color{black}	60.3	$\pm$	4.8	}
	&{\color{black}	64.5	$\pm$	3.3	}
	&{\color{black}	82.5	$\pm$	0.9	}
	&{\color{black}\underline{97.9	$\pm$	0.5}		}
	&{\color{black}	88.1	$\pm$	4.6	}
	&{\color{black}	98.9	$\pm$	1.7	}
	\\				
					
	&{\color{black}TriCL$^*$				}
	&{\color{black}\underline{77.6	$\pm$	1.1}		}
	&{\color{black}	OOM			}
	&{\color{black}\underline{76.7	$\pm$	1.3}		}
	&{\color{black}\underline{66.2	$\pm$	1.6}		}
	&{\color{black}\underline{73.2	$\pm$	0.8}		}
	&{\color{black}	80.8	$\pm$	0.5	}
	&{\color{black}	97.6	$\pm$	0.1	}
	&{\color{black}\underline{90.5	$\pm$	0.7}		}
	&{\color{black}\underline{99.7	$\pm$	0.1}		}
	\\

	\cmidrule{2-11}				
	&{\color{black}CHGNN				}
	&{\color{black}\textbf{78.9	$\pm$	0.7}		}
	&{\color{black}\textbf{89.6	$\pm$	1.5}		}
	&{\color{black}\textbf{79.8	$\pm$	1.0}		}
	&{\color{black}\textbf{69.5	$\pm$	1.1}		}
	&{\color{black}\textbf{75.3	$\pm$	0.8}		}
	&{\color{black}\textbf{83.7	$\pm$	0.2}		}
	&{\color{black}\textbf{98.1	$\pm$	0.1}		}
	&{\color{black}\textbf{91.9	$\pm$	0.2}		}
	&{\color{black}\textbf{99.8	$\pm$	0.1}		}
  \\
  \bottomrule
  \end{tabular}
  \label{table2}
\vspace{-5mm}
  \end{table*}

{\color{black}
\noindent \textit{\textbf{CA-Cora\footnote{https://github.com/malllabiisc/HyperGCN/tree/master/data/coauthorship/cora\label{ca-cora}}}} and \textit{\textbf{CA-DBLP\footnote{https://github.com/malllabiisc/HyperGCN/tree/master/data/coauthorship/dblp\label{ca-dblp}}}}  are co-authorship hypergraphs. Each node in a hypergraph represents a paper, and each hyperedge contains all papers by the same author.

\noindent \textit{\textbf{CC-Pubmed\footnote{https://github.com/malllabiisc/HyperGCN/tree/master/data/cocitation/pubmed\label{cc-pubmed}}}}, \textit{\textbf{CC-Citeseer\footnote{https://github.com/malllabiisc/HyperGCN/tree/master/data/cocitation/citeseer\label{cc-citeseer}}}} and \textit{\textbf{CC-Cora\footnote{https://github.com/malllabiisc/HyperGCN/tree/master/data/cocitation/cora\label{cc-cora}}}}   are co-citation hypergraphs. Each node in a hypergraph represents a  cited paper, and each hyperedge contains all publications cited by the same paper. Hyperedges with only one node are removed~\cite{yadati2018hypergcn}.

\noindent \textit{\textbf{20news\footnote{https://github.com/jianhao2016/AllSet/tree/main/data/raw\_data\label{allset}}}} is a news-word hypergraph extracted from the UCI Categorical Machine Learning Repository. Each node represents a piece of news, and each hyperedge represents all news containing the same word.

\noindent \textit{\textbf{ModelNet\textsuperscript{\ref{allset}}}} is a computer vision dataset extracted from Princeton CAD ModelNet40. Each node in the dataset  represents a CAD model. Each hyperedge connects five nearest neighbors of a node in a feature space, which is constructed by  Multi-view Convolutional Neural Network (MVCNN) and Group-View Convolutional Neural Network (GVCNN).

\noindent \textit{\textbf{NTU\textsuperscript{\ref{allset}}}} is a nearest-neighbor hypergraph extracted from a 3D dataset from National Taiwan University. 
Each node in the dataset represents a shape. Each  
 hyperedge connects five nearest neighbors of a node in a feature space which is constructed by MVCNN and GVCNN. 

\noindent \textit{\textbf{Mushroom\textsuperscript{\ref{allset}}}} is a mushroom-feature hypergraph extracted from the UCI Categorical Machine Learning
Repository. Each node represents a mushroom, and each hyperedge  contains all data with the same categorical features~\cite{chien2021you}.

{ CC-Cora, CC-Citeseer, and CC-Pubmed differ from Cora, Citeseer, and Pubmed as used elsewhere~\cite{pmlr-v48-yanga16}. Although CC-Cora and Cora, CC-Citeseer and Citeseer, and CC-Pubmed and Pubmed are generated from the same citation networks,  CC-Cora, CC-Citeseer, and CC-Pubmed are \textit{co-citation hypergraph datasets}, whereas Cora, Citeseer, and Pubmed are \textit{graph-structured citation datasets}. Co-citation hypergraph datasets represent mainly the co-citation relationships among multiple papers, while graph-structured citation datasets encode the direct and pairwise citation relationships between individual papers. }
\vspace{-2mm}
}

\vspace{-1mm}
\subsubsection{Competitors} 

We compare CHGNN with {nineteen} state-of-the-art proposals.

\noindent \textit{\textbf{MLP}}~\cite{rumelhart1985learning}  is a supervised learning framework that primarily focuses on encoding the features of node attributes without considering the underlying graph structure.

{\color{black}
\noindent\textit{\textbf{SimCLR}}~\cite{chen2020simple} is a self-supervised CL method. It adopts the 
augmenting-contrasting learning pattern, consisting of a base encoder and a projection head.

\noindent\textit{\textbf{MoCo}}~\cite{Moco} is also a self-supervised CL method. It incorporates a momentum-updated encoder and a negative sample queue.

\noindent\textit{\textbf{SimCLRv2}}~\cite{simclr2}  is an adaptation of SimCLR.  It involves pretraining a task-agnostic SimCLR model and then fine-tuning through distillation as a teacher model.

\noindent\textit{\textbf{Suave}}~\cite{Suave}  is a state-of-the-art semi-supervised CL method. It leverages cluster-based self-supervised methods.
}

\noindent\textit{\textbf{GCN}}~\cite{kipf2016semi} is a semi-supervised learning framework that employs a convolutional neural network architecture to perform node classification tasks on graphs.

 \noindent\textit{\textbf{GAT}}~\cite{velivckovic2017graph} utilizes a multi-head attention mechanism to dynamically generate aggregated weights for neighboring nodes.

 \noindent\textit{\textbf{DGI}}~\cite{velickovic2019deep} is a self-supervised graph representation learning framework that maximizes the mutual information between node representations and high-level summaries of the graph structure. 
 
 \noindent\textit{\textbf{SimGRACE}}~\cite{xia2022simgrace} is a GCL model that generates graph views by model parameter perturbation.

{\color{black}
 \noindent\textit{\textbf{HomoGCL}}~\cite{HomoGCL} leverages homophily in graph data to learn node representations more effectively.
    
 \noindent\textit{{\textbf{MA-GCL}}}~\cite{MAGCL} focuses on perturbing the architectures of GNN encoders rather than altering the graph input or model parameters.}

 \noindent\textit{\textbf{HyperConv}}\cite{bai2021hypergraph} is a graph representation learning framework that leverages higher-order relationships and local clustering structures for representation learning.

 \noindent\textit{\textbf{HGNN}}~\cite{feng2019hypergraph} is a graph representation learning model that generalizes and then approximates the convolution operation using truncated Chebyshev polynomials to enhance representation learning.

\noindent\textit{\textbf{HyperGCN}}~\cite{yadati2018hypergcn}  is a semi-supervised learning framework based on spectral theory.
 
\noindent\textit{\textbf{HNHN}}~\cite{dong2020hnhn} is a HyperGCN with nonlinear activation functions combined with a normalization scheme that can flexibly adjust the importance of hyperedges and nodes. 
 
\noindent\textit{\textbf{HyperSAGE}}~\cite{arya2020hypersage} is a graph representation learning framework that learns to embed hypergraphs by aggregating messages in a two-stage procedure.
 
\noindent\textit{\textbf{UniGNN}}~\cite{huang2021unignn} is a unified framework for interpreting the message-passing process in hypergraph neural networks. We report the results of UniSAGE, a variant of UniGNN that achieves the best average performance among all the variants in most cases~\cite{huang2021unignn}.
 
\noindent\textit{\textbf{AST}}~\cite{chien2021you} is the state-of-the-art supervised hypergraph-based model, which propagates information by a multiset function learned by Deep Sets~\cite{zaheer2017deep} and Set Transformer~\cite{lee2019set}.
 
\noindent\textit{\textbf{TriCL}}~\cite{lee2022m} is a state-of-the-art unsupervised contrastive hypergraph learning model that maximizes the mutual information between nodes, hyperedges, and groups.

MLP, SimCLR, MoCo, SimCLRv2, and Suave are non-graph based models, while GCN, GAT, DGI, SimGRACE, HomoGCL, and MA-GCL are graph based. The remaining baselines are hypergraph-based models. To enable the comparison of graph-based models with hypergraph-based models, we project hypergraphs to  2-projected graphs~\cite{yoon2020much,lee2022m}. Since unsupervised and supervised learning models exhibit similar performance in studies in the literature~\cite{zhu2021graph,lee2022m,MAGCL,HomoGCL}, we compare CHGNN with seven state-of-the-art unsupervised models: SimCLR, MoCo, DGI, SimGRACE, HomoGCL, MA-GCL, and TriCL.

\vspace{-2mm}
\subsubsection{Hyperparameters and Experimental Settings}

{\color{black}For the co-authorship and co-citation hypergraph datasets, we follow the data splits in the literature~\cite{yadati2018hypergcn,arya2020hypersage, huang2021unignn}. 
For 20news, ModelNet, NTU, and Mushroom, we label 50\% nodes for training and use the remaining nodes for testing, as in related work~\cite{chien2021you, 
cai2022hypergraph, 
kim2022equivariant,
yu2023hypergef}.
Table~\ref{table1} details the label ratios.}
Following existing studies~\cite{chien2021you,feng2019hypergraph,yadati2018hypergcn,arya2020hypersage,huang2021unignn}, we validate the performance of CHGNN and all baselines by performing classification and report the average node classification accuracy. In particular, 10-fold cross-validation is conducted on each dataset.

We set the hyperparameters of the baselines as suggested by the respective papers.  For CHGNN, we set the number of layers in the H-HyperGNN to 2 and use an MLP consisting of two fully connected layers with 16 neurons in the hidden layer.  For  the node classification task, we set $p_{node}=0.2$, $p_{\tau }=0.8$, $\varepsilon_c = 0.2 $, $\varepsilon_e = 0.2 $, $\lambda_h = 1 $, and $\lambda_c = 1 $. Other hyperparameters are listed in Table~\ref{table_hp} where nhid is the embedding dimensionality of the hidden layer and nproj is the projection dimensionality. The source code is available online\footnote{https://github.com/yumengs-exp/CHGNN}.  All experiments are conducted on a server with  an Intel(R) Xeon(R) W-2155 CPU, 128GB memory, and an NVIDIA TITAN RTX GPU.
{
\begin{table*}[t]
\renewcommand\tabcolsep{5pt}
\centering
\caption{Classification accuracy with different label ratios (\%) on CC-Cora and NTU, where (i) $*$ indicates unsupervised learning methods and (ii) the best and second-best performances are marked in bold and underlined, respectively.}

\begin{tabular}{llcccccccccccccccccc}
  \toprule
  
  &\textbf{Dataset}
  &\multicolumn{9}{c}{CC-Cora}
  &\multicolumn{8}{c}{\color{black}NTU}
  \\
  \cmidrule(lr){1-2}  \cmidrule(lr){3-11} \cmidrule(lr){12-19}

  &\textbf{Ratio} 
  & \textbf{0.5\%} 
  & \textbf{1\%}   
  & \textbf{2\%}  
  &\textbf{3\%}
  &\textbf{10\%} 
  &\textbf{20\%}  
  &\textbf{30\%} 
  &\textbf{40\%}
  &\textbf{50\%}
  & \textbf{\color{black}3\%} 
  & \textbf{\color{black}4\%}   
  & \textbf{\color{black}5\%}  
  &\textbf{\color{black}10\%}
  &\textbf{\color{black}20\%} 
  &\textbf{\color{black}30\%}  
  &\textbf{\color{black}40\%} 
  &\textbf{\color{black}80\%}
  \\
  
   \midrule
   \multirow{5}{*}{\rotatebox{90}{\textbf{Non-Graph}}}
   &MLP
   & 31.1
   & 39.9
   & 45.7
   & 52.6
   & 63.8
   & 68.5
   & 70.5
   & 72.7
   & 73.2
   & {\color{black}42.0}
   & {\color{black}50.2}
   & {\color{black}55.3}
   & {\color{black}65.1}
   & {\color{black}77.9}
   & {\color{black}79.9}
   & {\color{black}84.9}
   & {\color{black}91.1}
   \\
  
   &SimCLR$^*$
   & 30.1
   & 39.6
   & 44.7
   & 47.1
   & 54.0
   & 56.0
   & 56.8
   & 57.2
   & 59.9
   & {\color{black}31.5}
   & {\color{black}31.3}
   & {\color{black}35.7}
   & {\color{black}56.9}
   & {\color{black}58.3}
   & {\color{black}59.6}
   & {\color{black}77.2}
   & {\color{black}82.9}
   \\
  
   &MoCo$^*$
   & 24.3
   & 37.3
   & 41.0
   & 43.5
   & 50.5
   & 52.6
   & 53.3
   & 54.5
   & 55.5
   & {\color{black}31.2}
   & {\color{black}31.1}
   & {\color{black}32.1}
   & {\color{black}51.4}
   & {\color{black}51.6}
   & {\color{black}61.3}
   & {\color{black}77.6}
   & {\color{black}84.2}
   \\
  
   &SimCLRv2
   & 30.3
   & 39.1
   & 45.0
   & 48.9
   & 64.0
   & 68.5
   & 70.6
   & 72.8
   & 73.7
   & {\color{black}36.7}
   & {\color{black}50.8}
   & {\color{black}55.6}
   & {\color{black}69.6}
   & {\color{black}78.6}
   & {\color{black}80.2}
   & {\color{black}81.3}
   & {\color{black}87.1}
   \\
  
   &Suave
   & 37.7
   & 42.1
   & 53.9
   & 58.3
   & 64.4
   & 65.9
   & 68.9
   & 67.5
   & 68.7
   & {\color{black}50.4}
   & {\color{black}54.9}
   & {\color{black}61.0}
   & {\color{black}73.5}
   & {\color{black}80.4}
   & {\color{black}83.8}
   & {\color{black}86.8}
   & {\color{black}92.4}
   \\

     \midrule
     \multirow{6}{*}{\rotatebox{90}{\textbf{Graph}}}
     &GCN
     & 21.3
     & 36.7
     & 47.2
     & 53.5
     & 66.4
     & 69.7
     & 72.0
     & 73.2
     & 74.3
     & {\color{black}46.0}
     & {\color{black}50.6}
     & {\color{black}53.0}
     & {\color{black}60.7}
     & {\color{black}64.4}
     & {\color{black}65.6}
     & {\color{black}67.0}
     & {\color{black}90.4}
     \\
     &GAT
     & 24.8
     & 39.4
     & 49.0
     & 53.4
     & 66.0
     & 70.9
     & 72.1
     & 73.0
     & 74.3
     & {\color{black}45.8}
     & {\color{black}53.4}
     & {\color{black}53.3}
     & {\color{black}59.4}
     & {\color{black}63.8}
     & {\color{black}64.9}
     & {\color{black}66.4}
     & {\color{black}88.5}
     \\
     &DGI$^*$
     & 26.8
     & 30.2
     & 38.8
     & 43.3
     & 62.0
     & 63.8
     & 63.6
     & 63.2
     & 62.3
     & {\color{black}47.6}
     & {\color{black}56.3}
     & {\color{black}62.2}
     & {\color{black}65.4}
     & {\color{black}73.3}
     & {\color{black}74.5}
     & {\color{black}79.8}
     & {\color{black}91.2}
     \\
     &SimGRACE$^*$
& 24.9
& 31.4
& 37.1
& 41.4
& 62.9
& 62.1
& 61.8
& 62.6
& 62.7
& {\color{black}48.9}
& {\color{black}56.6}
& {\color{black}61.5}
& {\color{black}65.4}
& {\color{black}73.6}
& {\color{black}77.7}
& {\color{black}78.1}
& {\color{black}87.2}
\\
&HomoGCL$^*$
& 32.2
& 37.7
& 42.0
& 45.1
& 51.4
& 53.0
& 54.4
& 55.5
& 56.5
& {\color{black}50.8}
& {\color{black}50.8}
& {\color{black}57.2}
& {\color{black}58.8}
& {\color{black}64.9}
& {\color{black}67.6}
& {\color{black}70.9}
& {\color{black}84.1}
\\
&MA-GCL$^*$
& 45.5
& 51.0
& 54.9
& 57.4
& 62.1
& 63.5
& 64.6
& 65.6
& 66.5
&{\color{black}\underline{56.9}}
& {\color{black}56.9}
& {\color{black}64.3}
& {\color{black}69.0}
& {\color{black}74.0}
& {\color{black}76.3}
& {\color{black}78.1}
& {\color{black}85.6}
\\
\midrule
\multirow{10}{*}{\rotatebox{90}{\textbf{Hypergraph}}}
&HyperConv
& 27.2
& 30.8
& 36.6
& 43.2
& 50.0
& 55.9
& 54.2
& 62.5
& 66.1
& {\color{black}55.6}
& {\color{black}57.7}
&{\color{black}\underline{64.7}}
& {\color{black}67.4}
& {\color{black}81.3}
& {\color{black}83.1}
& {\color{black}84.6}
& {\color{black}90.1}
\\
&HGNN
& 30.2
& 40.8
& 56.7
& 60.4
& 71.6
& 75.8
& 77.1
& 77.3
& 78.6
& {\color{black}28.3}
& {\color{black}31.4}
& {\color{black}29.5}
& {\color{black}38.6}
& {\color{black}43.6}
& {\color{black}45.6}
& {\color{black}48.7}
& {\color{black}89.2}
\\
&HNHN
& 31.7
& 44.9
& 52.6
& 56.2
& 68.0
& 71.1
& 72.6
& 74.6
& 74.9
& {\color{black}43.2}
& {\color{black}47.4}
& {\color{black}51.1}
& {\color{black}65.4}
& {\color{black}74.7}
& {\color{black}78.0}
& {\color{black}79.5}
& {\color{black}88.5}
\\
&HyperSAGE
& 39.9
& 47.0
& 54.9
& 58.4
& 68.9
& 73.4
& 76.4
& 77.0
& 77.7
& {\color{black}41.5}
& {\color{black}47.1}
& {\color{black}52.8}
& {\color{black}62.8}
& {\color{black}72.5}
& {\color{black}77.8}
& {\color{black}78.2}
& {\color{black}87.0}
\\
&UniGNN
& 38.5
& 49.0
& 58.5
& 62.5
& 71.7
& 74.5
& 76.8
& 77.1
& 78.4
& {\color{black}47.8}
& {\color{black}55.5}
& {\color{black}58.7}
& {\color{black}71.6}
& {\color{black}79.1}
& {\color{black}82.2}
& {\color{black}84.4}
& {\color{black}86.5}
\\		
&AST
& 48.6
& 49.8
& 55.5
& 58.9
& 69.5
& 71.4
& 75.2
& 77.9
& 78.4
& {\color{black}40.5}
& {\color{black}44.5}
& {\color{black}52.5}
& {\color{black}66.0}
& {\color{black}75.3}
& {\color{black}75.4}
& {\color{black}78.9}
& {\color{black}83.9}
\\  
&TriCL$^*$
&\underline{53.3}
&\underline{66.6}
&\underline{69.8}
&\underline{73.5}
&\textbf{76.9}
&\underline{78.6}
&\underline{78.2}
&\underline{78.5}
&\underline{79.0}
& {\color{black}53.6}
&{\color{black}\underline{60.8}}
& {\color{black}63.9}
&{\color{black}\underline{76.4}}
&{\color{black}\underline{83.4}}
&{\color{black}\underline{86.5}}
&{\color{black}\underline{87.9}}
&{\color{black}\underline{92.5}}
\\
\cmidrule{2-19}
&CHGNN
&\textbf{64.8}
&\textbf{66.7}
&\textbf{72.2}
&\textbf{73.9}
&\textbf{76.9}
&\textbf{78.7}
&\textbf{78.5}
&\textbf{79.7}
&\textbf{79.9}
&{\color{black}\textbf{57.7}}
&{\color{black}\textbf{63.8}}
&{\color{black}\textbf{68.8}}
&{\color{black}\textbf{77.8}}
&{\color{black}\textbf{84.8}}
&{\color{black}\textbf{86.8}}
&{\color{black}\textbf{88.3}}
&{\color{black}\textbf{94.1}}
\\
  \bottomrule
  \end{tabular}
  \label{table4}
  \vspace{-6mm}
  \end{table*}}

\vspace{-3mm}
\subsection{Comparison}  
\label{comparison}
{\color{black}
Table~\ref{table2} reports the accuracy of all methods. CHGNN tops consistently across all datasets, outperforming the closest competitors by 1.3\% on CA-Cora, 0.6\% on CA-DBLP, 3.1\% on CC-Pubmed, 3.3\% on CC-Citeseer, 2.1\% on CC-Cora, 0.4\% on 20news, 0.2\% on ModelNet, 1.4\% on NTU, and 0.1\% on Mushroom.  This is because CHGNN is capable of exploiting hypergraph structural information for enhanced node representation and learning. Moreover, there are three notable observations as follows.

First, 
while Suave, a state-of-the-art non-graph based method, surpasses  MLP, SimCLR, Moco, and SimCLRv2, it is less accurate than CHGNN. This is due to CHGNN 's ability to exploit hypergraph structural information for enhanced node representation and learning. 

Next, the graph-based methods  GCN, GAT, DGI, SimGRACE, HomoGCL, and MA-GCL require the conversion of hypergraphs into graphs, which causes a loss of structural information, making them less effective.
MA-GCL achieves higher accuracy than HomoGCL on six datasets (CA-Cora, CC-Citeseer, CC-Cora, 20news, NTU, and Mushroom), due to their lower homophily in graph transformations.
HyperGNNs, however, are able to learn richer information from hypergraphs by directly incorporating structural information into the learning process, leading to more effective and accurate learning outcomes.

\noindent{\textit{\textbf{Discussion.}} Graph and hypergraph models have fundamental differences. Hypergraph models focus on classifying data based on multiple correlations and often ignore pairwise relationships~\cite{feng2019hypergraph, yadati2018hypergcn,arya2020hypersage,huang2021unignn,chien2021you,lee2022m}. In contrast, graph models concentrate on pairwise relationships and tend to disregard multiple correlations~\cite{kipf2016semi,velivckovic2017graph,velickovic2019deep,xia2022simgrace}. Thus, on datasets with predominant multiple correlations, e.g., CC-Cora, CC-Citeseer, and CC-Pubmed, graph models generally underperform.  This distinction highlights that these models are not interchangeable.

Second,  TriCL performs comparably to all supervised baselines, even without any labeled data. TriCL outperforms the state-of-the-art AST and UniGNN on CA-Cora, CC-Pubmed, CC-Citeseer, CC-Cora, NTU, and Mushroom. This indicates that CL is capable of exploiting information from unlabeled nodes. Compared with TriCL, CHGNN (i) incorporates contrastive loss in training to enhance supervised learning and (ii) provides enough variance to facilitate CL via its view generator.

Third, although AST can perform well with sufficient training data (e.g., 50\% labeled data~\cite{chien2021you}), it is not well-suited in our target application scenario with limited labeled data. In contrast, CHGNN requires relatively little labeled data to achieve high accuracy. This is due in part to its use of CL, which can exploit information from unlabeled data.

}

\vspace{-4mm}
\subsection{Effect of Label Ratios}

{\color{black}
We study the impact of splits on CC-Cora and NTU.
Table \ref{table4} shows the results.  CHGNN performs better than the competitors in all settings. On CC-Cora, (i) the performance of CHGNN trained on 3\% of the labeled nodes is comparable to that of UniGNN trained on 10\%; (ii) in sparse training data scenarios, e.g., when only 0.5\% of the nodes are labeled, CHGNN's advantage over the best competitor, TriCL,  is at 21.5\%; (iii) CHGNN  outperforms the best semi-supervised competitor, AST,  by  1.9\%, even when 50\% of the nodes are labeled. 
On NTU, (i) the accuracy of CHGNN trained on 30\% labeled nodes is comparable to that of UniGNN trained on 80\%;  (ii) when just 5\% of the nodes are labeled, CHGNN's advantage against the best competitor, HyperConv,  is at 4.1\%.
 We attribute this to the CL component of CHGNN, which mines topological and feature information through constructed positive and negative samples and enables it to learn better from unlabeled data. Results on the remaining seven datasets exhibit similar patterns. 
 
Next, the classification accuracy of the semi-supervised methods increases linearly with the label ratio, while the performance of the CL methods stabilizes when the label ratio exceeds 20\%.
This is because CL methods ignore supervised information during training. The node embedding distribution generated by CL does not match the distribution of the labels well. Thus, it is difficult for the downstream classifier to accurately identify class boundaries based on the training set. In addition, CHGNN outperforms the semi-supervised competitors, due to its use of a contrastive loss for embedding learning.
}

\vspace{-0.4cm}
{\color{black}\subsection{Effect of Imbalance Labels}
We compare CHGNN with the state-of-the-art hypergraph neural network models UniGNN~\cite{huang2021unignn}, AST~\cite{chien2021you}, and TriCL~\cite{lee2022m} on datasets CC-Cora and CC-Citeseer with imbalanced labels.

Following the imbalanced setting in a previous study~\cite{juan2023ins},  we randomly select (approximately) half of the classes as the minority classes and the rest as the majority classes. For the seven classes of CC-Cora, classes 4, 5, and 6 are the minority classes, while classes 1, 2, 3, and 7 are the majority classes. For the six classes of CC-Citeseer, classes 1, 2, and 6 are the minority classes, and classes 3, 4, and 5 are the majority classes. We then randomly sample 10 nodes from each minority class and 20 nodes from the majority classes to form the training set. 

Fig.~\ref{fig_imbalance} reports the classification accuracy of the four methods for each class. For the minority classes, CHGNN and TriCL, which use contrastive loss, exhibit commendable performance. In contrast,  their effectiveness is lower for the majority classes compared to AST and UniGNN, which do not employ contrastive loss. This indicates that contrastive loss benefits minority classes more than it benefits majority classes. For the majority classes, the combination of contrastive loss and classification loss biases the node distribution of CHGNN, leading to reduced effectiveness. 
In short, CHGNN slightly suffers on imbalanced datasets where majority classes dominate.

Note that datasets with highly skewed labels are rare in real-world scenarios~\cite{yadati2018hypergcn, huang2021unignn}. This implies that CHGNN exhibits performance superiority in most cases (see Section~\ref{comparison}). Given the limitation on imbalanced data, we plan to investigate the impact of jointly training contrastive learning losses and semi-supervised learning losses on node embedding distribution. We leave this for future work.
\begin{figure}
\centering
         \subfloat[ {CC-Cora}]{
          \includegraphics[width=3.6 cm]{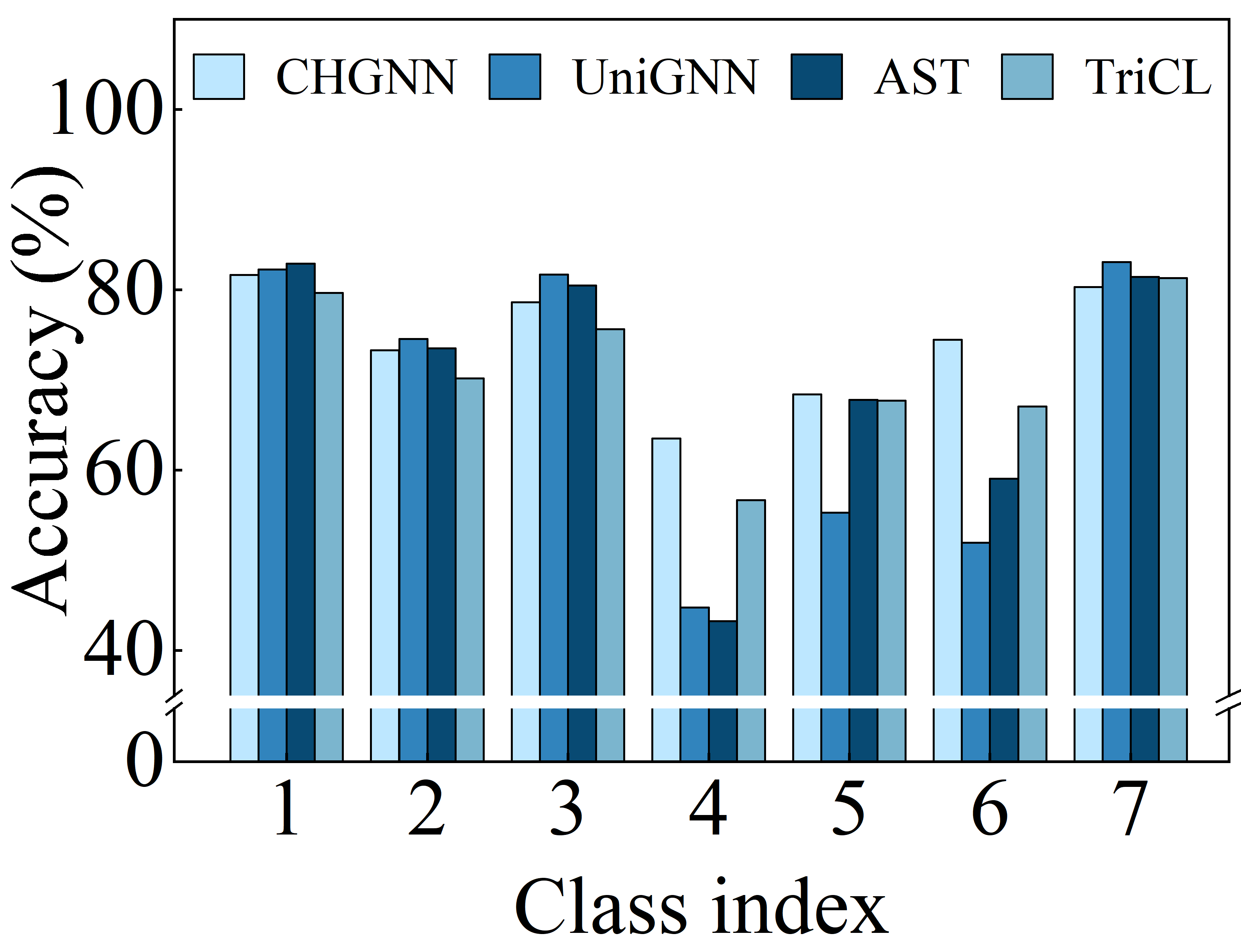}}   
           \quad\quad\quad
        \subfloat[ {CC-Citeseer}]{
          \includegraphics[width=3.6 cm]{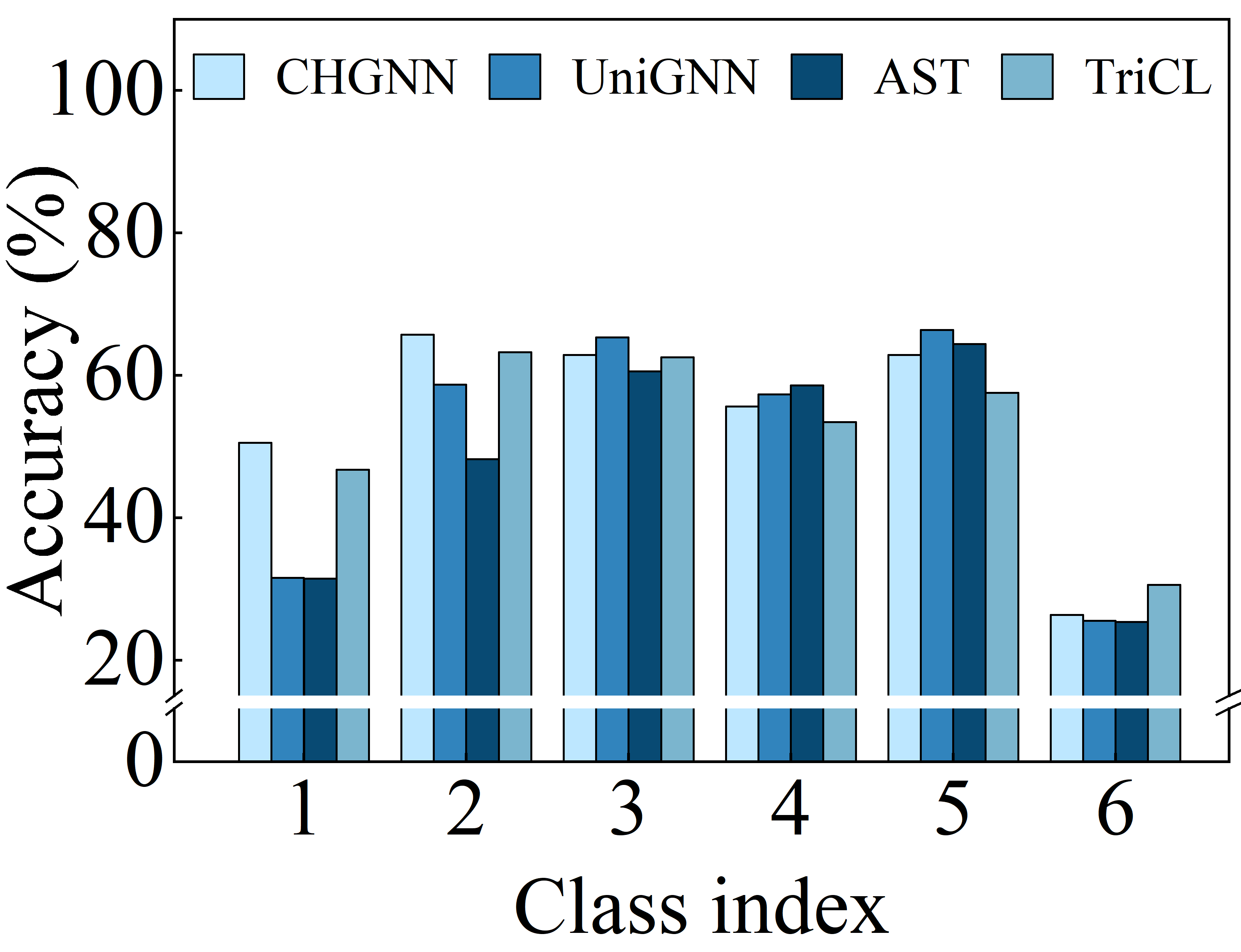}}

           \caption{The impact of imbalanced labels.}
     \label{fig_imbalance}
     \vspace{-5mm}
   \end{figure}
  }

\begin{table*}[t]
  \renewcommand\tabcolsep{8pt}
  \centering
  \caption{\color{black}Ablation study results, where (i) values after $\pm$ are standard deviations; (ii) the best performances are marked in bold; (iii) $\mathcal{L}_{h}$, $\mathcal{L}_{cl}$, and $\mathcal{L}_{crocl}$ represent hyperedge homogeneity loss, basic contrastive loss, and cross-validation contrastive loss, respectively; (iv) ``Aug'' indicates the strategy of augmenting the views, which can be no CL (``-''), random augmentation (RandAug), and adaptive hypergraph view generator (ViewGen); (v) ``Encoder'' indicates the model used for encoding hypergraphs, which can be HyperGNN and H-HyperGNN; and (vi) $V_i\,(i=A,\cdots,P)$ denotes variants.}
  \begin{tabular}{c|c|c|c|c|c|ccccc}
  \toprule

 &
 $\mathcal{L}_{h}$
 &$\mathcal{L}_{cl}$
 &$\mathcal{L}_{crocl}$
 &\textbf{Aug}
 &\textbf{Encoder}                    
 &\textbf{CA-Cora}  
 &\textbf{CA-DBLP} 
 &\textbf{CC-Pubmed}    
 &\textbf{CC-Citeseer} 
 &\textbf{CC-Cora}  \\
 
 \midrule

$V_A$
 &
- 
&- 
&- 
&- 
&HyperGNN        
& 74.0  $\pm$  0.4
& 87.9  $\pm$  0.1
& 75.1  $\pm$  0.8
& 62.0  $\pm$  0.8
& 66.8  $\pm$  1.5 \\

$V_B$
&
\checkmark 
&- 
&- 
&- 
&HyperGNN                            
& 73.1  $\pm$  0.7
& 87.1  $\pm$  0.1
& 74.9  $\pm$  0.3
& 60.3  $\pm$  0.7
& 64.8  $\pm$  1.3  \\

$V_C$
&
- 
&- 
&- 
&- 
&H-HyperGNN  
& 74.5  $\pm$  0.1
& 88.1  $\pm$  0.1
& 75.4  $\pm$  0.9
& 61.6  $\pm$  1.3
& 67.3  $\pm$  1.3\\

$V_D$
&
\checkmark 
&- 
&- 
&- 
&H-HyperGNN                            
& 74.7  $\pm$  0.1
& 88.2  $\pm$  0.1
& 75.9  $\pm$  1.1
& 62.5  $\pm$  0.2
& 67.4  $\pm$  1.2  \\

$V_E$
&
- 
&\checkmark 
&- 
&RandAug 
&HyperGNN                  
& 71.8  $\pm$  1.6
& 88.5  $\pm$  0.2
& 74.6  $\pm$  1.6
& 59.1  $\pm$  3.2
& 63.1  $\pm$  2.6  \\

$V_F$
&
- 
& - 
&\checkmark 
&RandAug 
&HyperGNN                       
& 75.5  $\pm$  0.8
& 88.7  $\pm$  0.2
& 75.8  $\pm$  3.6
& 68.2  $\pm$  1.7
& 72.0  $\pm$  1.6  \\

$V_G$
&
-  
& \checkmark 
&\checkmark 
&RandAug 
&HyperGNN             
& 74.0  $\pm$  0.9
& 88.6  $\pm$  0.2
& 75.4  $\pm$  4.0
& 67.7  $\pm$  1.1
& 70.5  $\pm$  2.2   \\

$V_H$
&
\checkmark 
& \checkmark 
&- 
&RandAug 
&H-HyperGNN            
& 70.3  $\pm$  2.3
& 88.6  $\pm$  0.2
& 74.1  $\pm$  1.7
& 58.9  $\pm$  1.9
& 63.7  $\pm$  1.3 \\

$V_I$
&
\checkmark 
& - 
&\checkmark 
&RandAug 
&H-HyperGNN            
& 75.3  $\pm$  1.2
& 88.8  $\pm$  0.2
& 76.7  $\pm$  3.4
& 67.5  $\pm$  1.5
& 71.5  $\pm$  0.6  \\

$V_J$
&
\checkmark 
&\checkmark 
&\checkmark 
&RandAug 
&H-HyperGNN    
& 74.5  $\pm$  1.9
& 88.5  $\pm$  0.3
& 76.1  $\pm$  2.6
& 67.0  $\pm$  2.3
& 69.9  $\pm$  3.8 \\

$V_K$
&
- 
&\checkmark 
&- 
& ViewGen 
&HyperGNN           
& 75.7  $\pm$  1.0
& 89.2  $\pm$  0.1
& 75.5  $\pm$  1.5
& 62.6  $\pm$  1.5
& 66.3  $\pm$  2.6   \\

$V_L$
&
- 
& - 
&\checkmark 
&ViewGen 
&HyperGNN                       
& 76.5  $\pm$  1.3
& 89.1  $\pm$  0.2
& 77.9  $\pm$  1.7
& 67.6  $\pm$  1.8
& 73.1  $\pm$  0.9   \\

$V_M$
&
-  
& \checkmark 
&\checkmark 
&ViewGen 
&HyperGNN             
& 76.8  $\pm$  1.4
& 89.1  $\pm$  0.1
& 77.7  $\pm$  1.8
& 67.9  $\pm$  1.2
& 72.6  $\pm$  1.3 \\

$V_N$
&
\checkmark 
& \checkmark 
&- 
&ViewGen 
&H-HyperGNN            
& 78.1  $\pm$  1.2
& 89.1  $\pm$  0.2
& 75.3  $\pm$  1.3
& 62.1  $\pm$  1.2
& 71.0  $\pm$  2.4  \\
      
$V_O$
&     
\checkmark 
& - 
&\checkmark 
&ViewGen 
&H-HyperGNN            
& 77.5  $\pm$  1.0
& 89.1  $\pm$  0.2
& 76.4  $\pm$  1.0
& 68.0  $\pm$  1.8
& 74.3  $\pm$  0.7  \\

$V_P$
&
\checkmark 
&\checkmark 
&\checkmark 
&ViewGen 
&H-HyperGNN    
& \textbf{78.3 $\pm$ 1.0 }     
& \textbf{89.3 $\pm$ 0.2 }   
& \textbf{78.3 $\pm$ 2.9 }     
& \textbf{68.5 $\pm$ 1.7 }  
& \textbf{75.5 $\pm$ 0.9 } \\

      \bottomrule
      \end{tabular}
      \vspace{-5mm}
      \label{table_ablation}
      \end{table*}

\begin{figure}[t]
  \centering 
  \vspace{-3mm}
      \subfloat[ {\color{black}$\lambda _{c}$ and $\lambda _{e}$ (CC-Cora)}]{
         \includegraphics[width=4.1 cm]{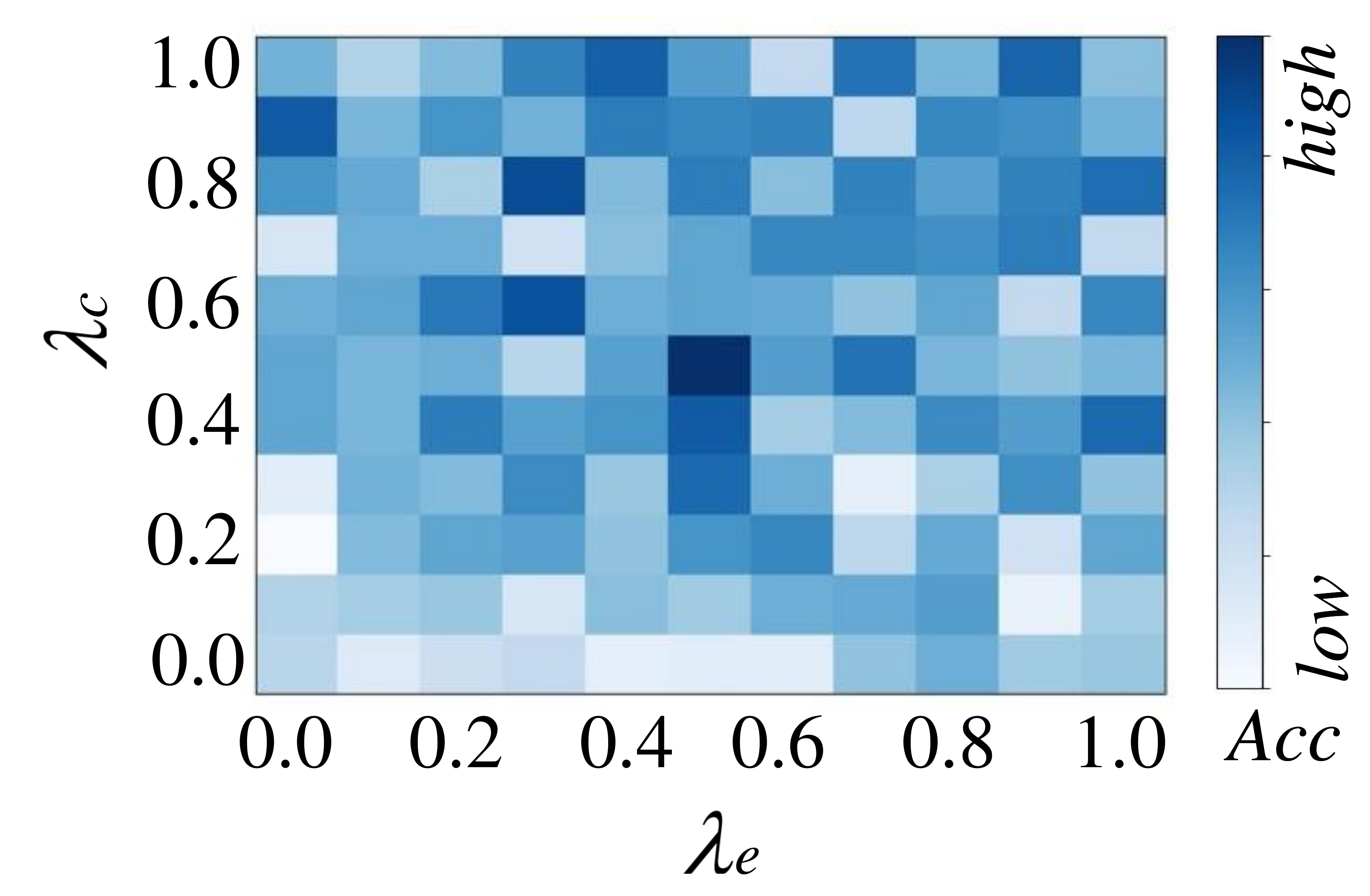}} 
         \quad
          \subfloat[ {\color{black}$\varepsilon_c$ and $\varepsilon_e$ (CC-Cora)} ]{
         \includegraphics[width=4.1 cm]{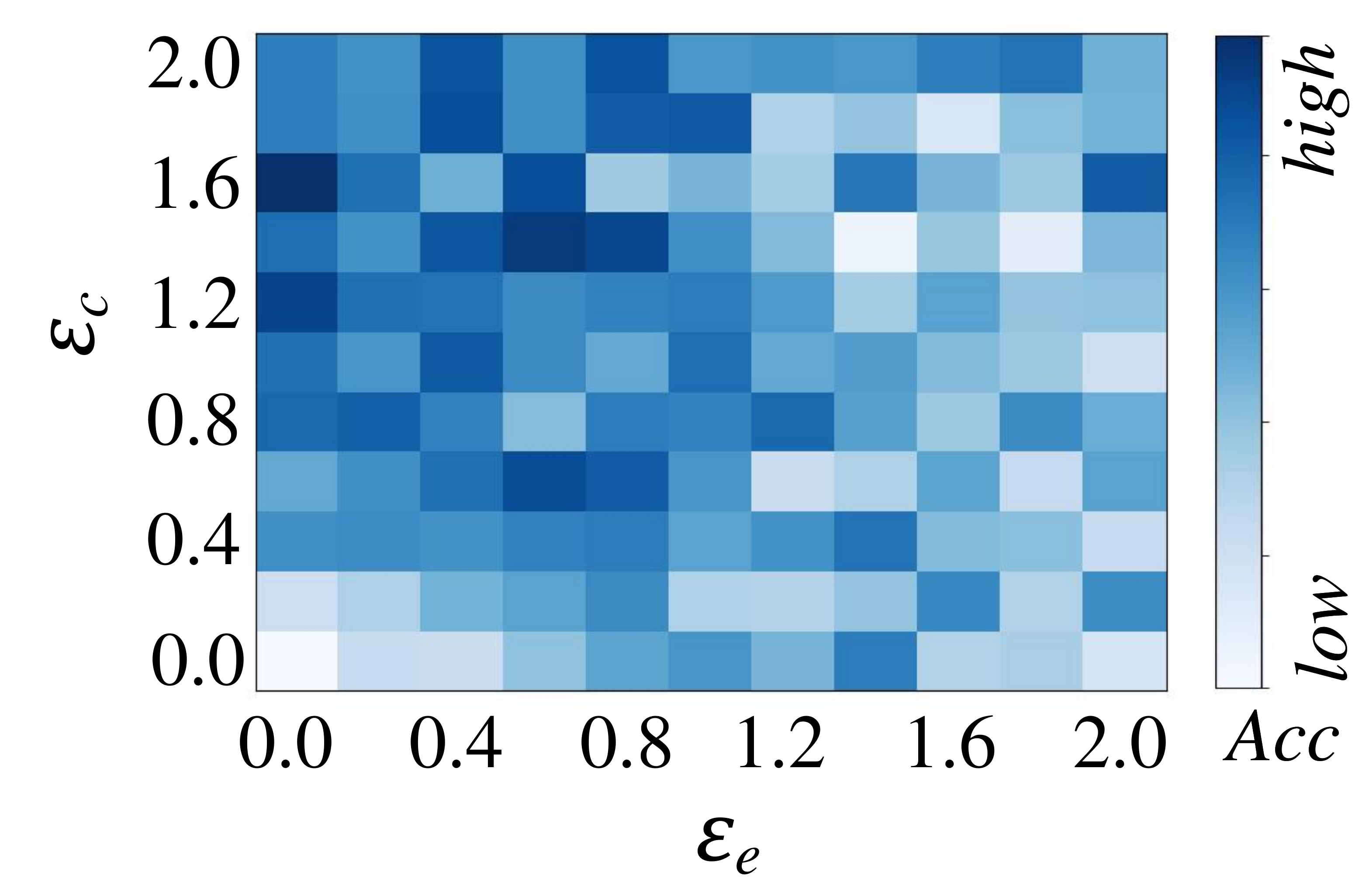}}

          \vspace{-3mm}
        \subfloat[Impact of  $\lambda _{h}$]{
         \includegraphics[width=3.8 cm]{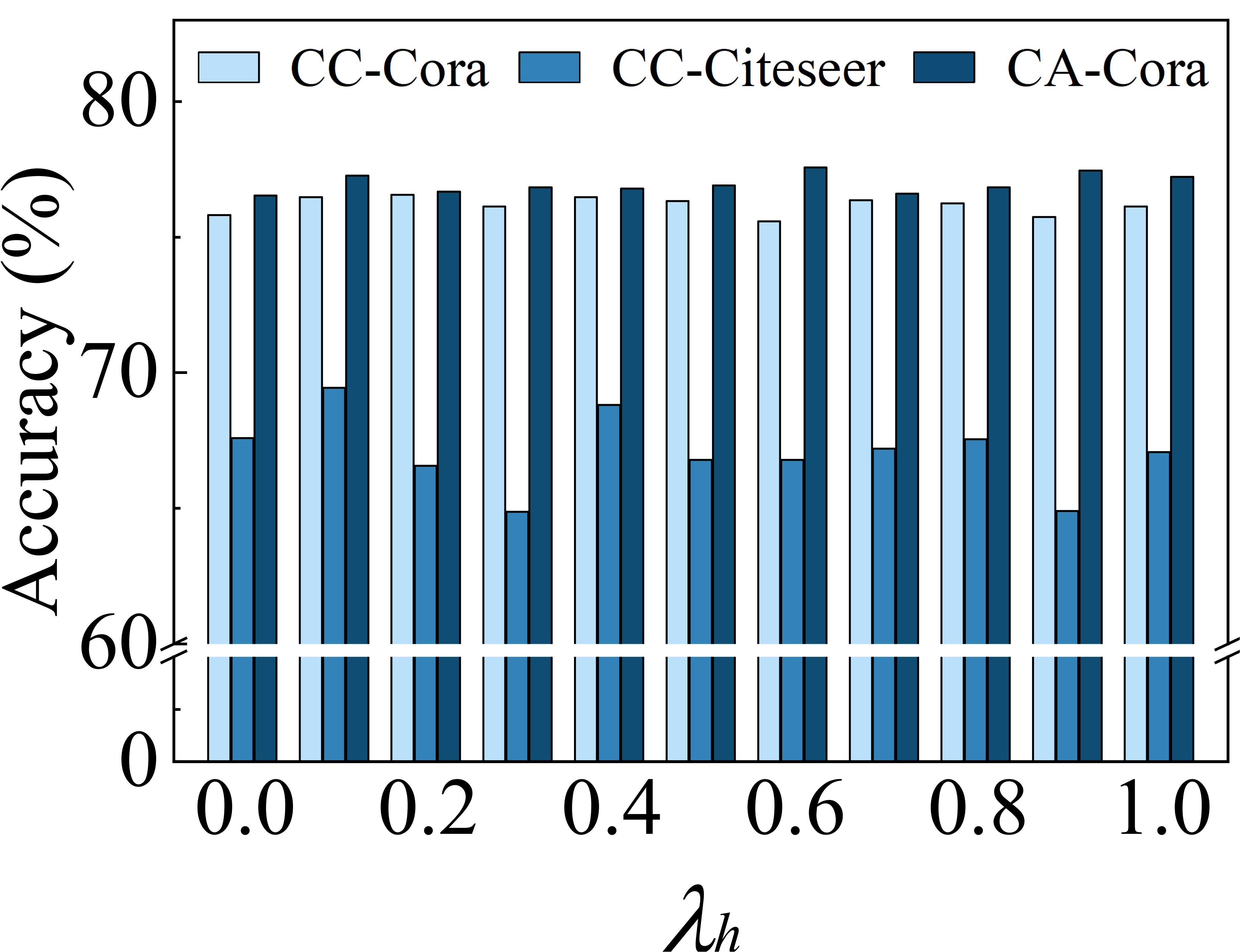}} 
         \quad \quad 
         \subfloat[ Impact of  $\lambda _{nc}$ ]{
          \includegraphics[width=3.8 cm]{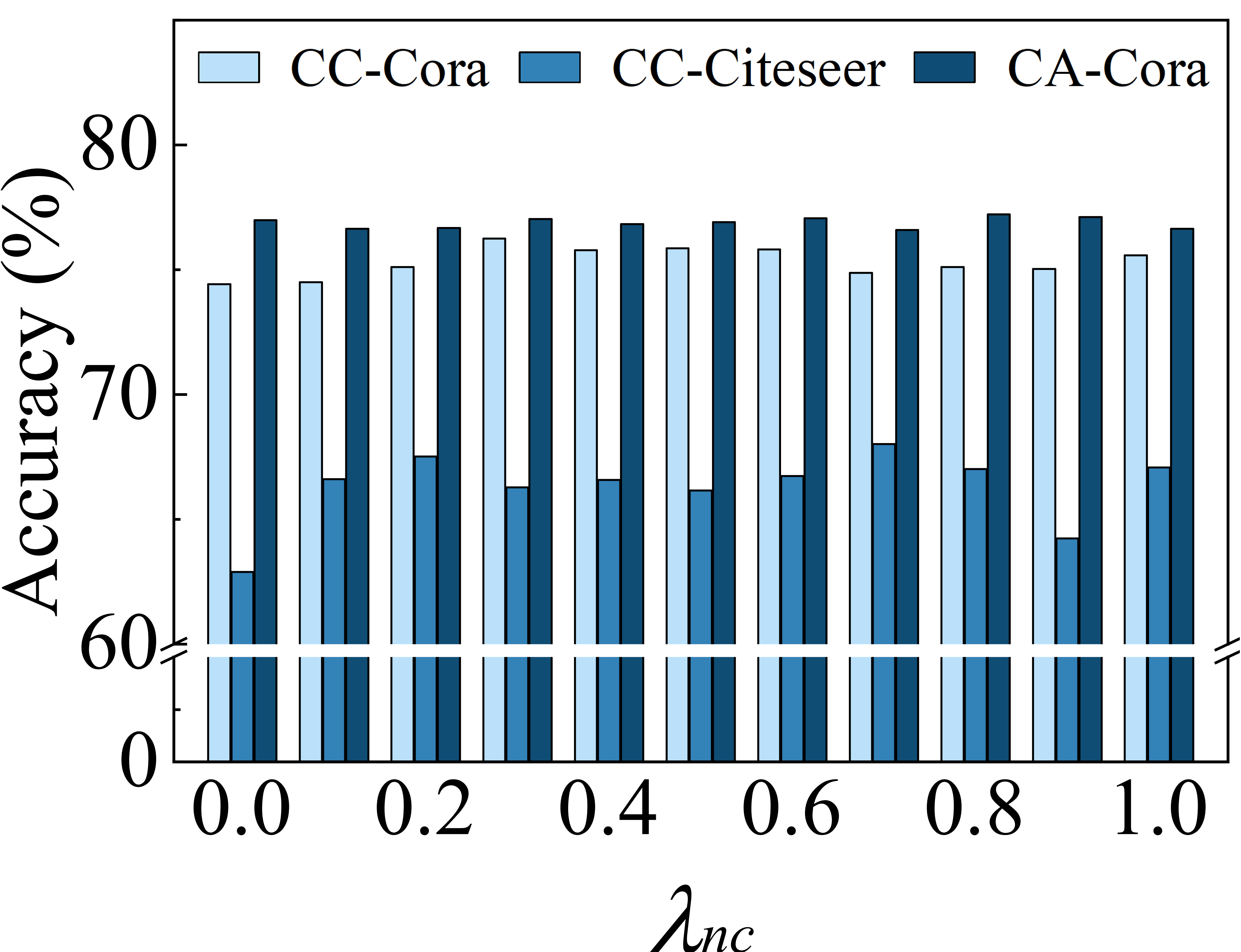}}  
          
         \vspace{-3mm}
          \subfloat[ Impact of $\lambda _{ne}$]{
         \includegraphics[width=3.8cm]{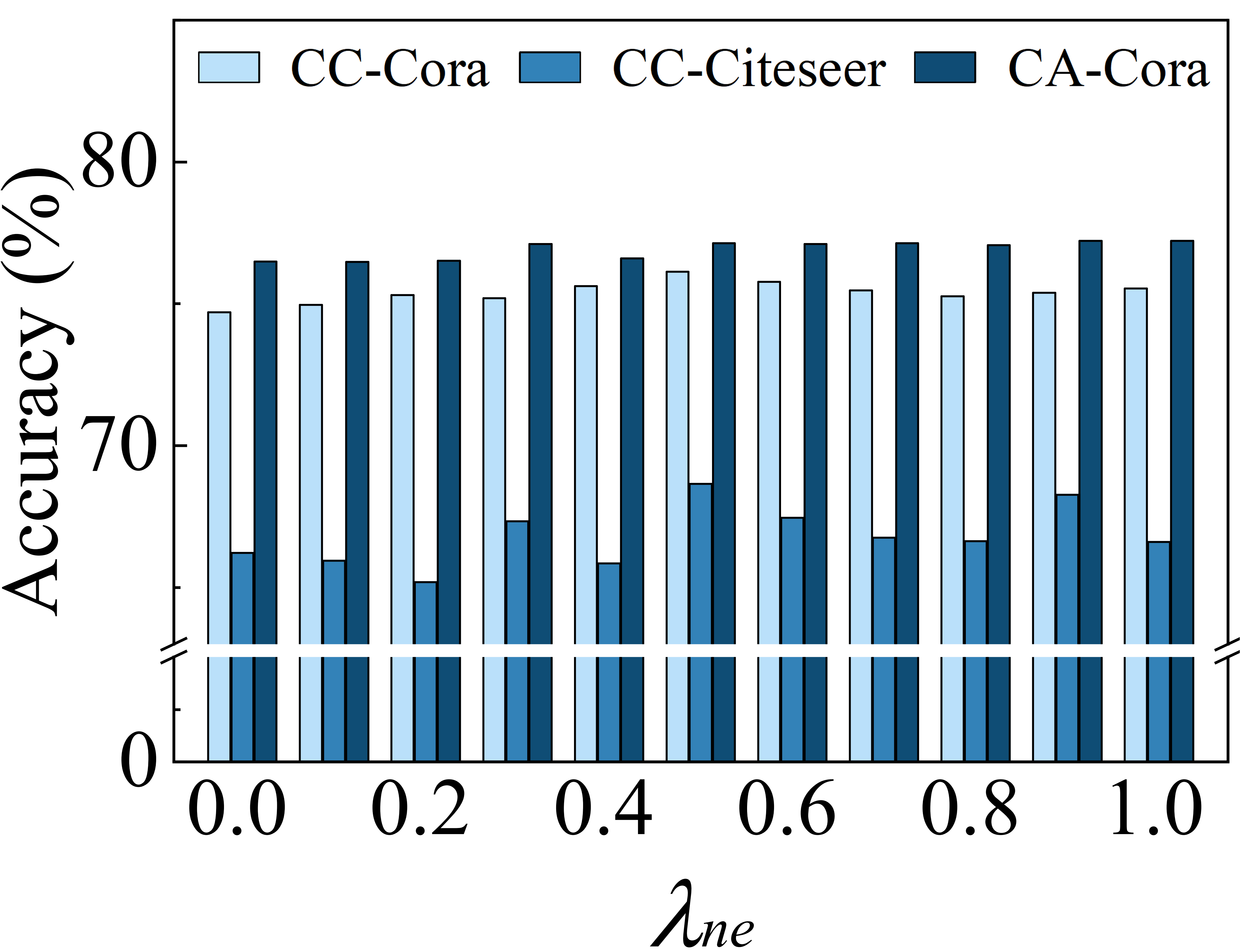}} 
         \quad\quad 
         \subfloat[ Impact of  $\lambda _{ec}$]{
          \includegraphics[width=3.8cm]{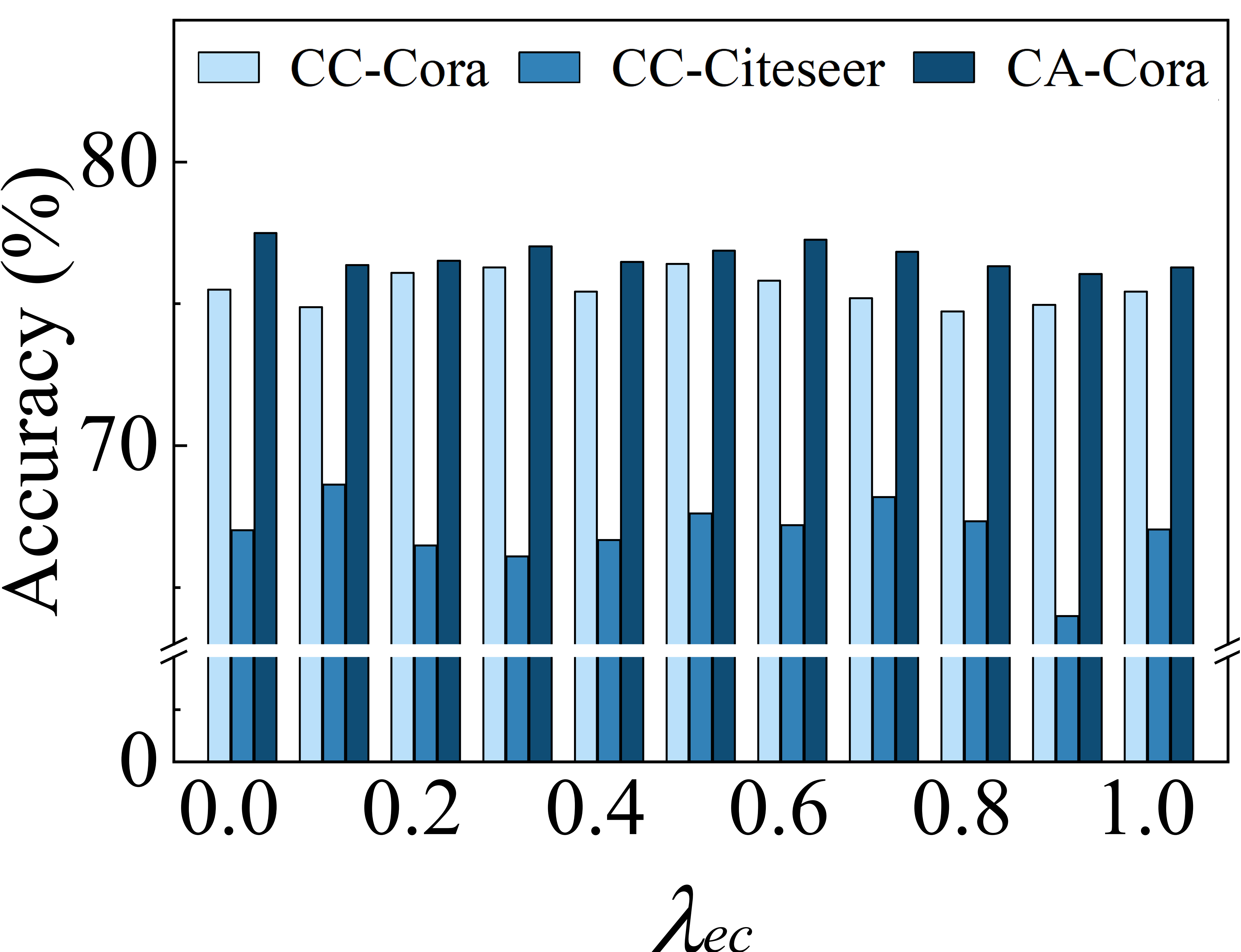}} 
   
     \caption{ \color{black}Impact of hyperparameters, where lighter colors in heatmaps indicate lower accuracy. }
     \label{fig_hp} 
\vspace{-8mm}
    \end{figure}

\vspace{-4mm}
\subsection{Parameter Analysis}
\textit{\textbf{Effect of $\lambda _{c}$ and $\lambda _{e}$.}}
We evaluate the impact of $\lambda _{c}$ and $\lambda _{e}$ on classification, where $\lambda _{c}$ and $\lambda _{e}$ control the basic contrastive losses (see Section~\ref{section:4.4.2}). Figs. \ref{fig_hp} (a) is  the heatmap showing the classification accuracy on CC-Cora,  where both $\lambda _{c}$ and $\lambda _{e}$ are varied from 0 to 1. 
CHGNN achieves the highest accuracy when $0.3<\lambda_{c}<0.5$ and $\lambda_{e}=0.5$. This is because two basic contrastive losses interact with each other, which necessitates the adjustments of $\lambda _{c}$ and $\lambda _{e}$ jointly for optimal performance.

\noindent
\textit{\textbf{Effect of $\varepsilon_{c}$ and  $\varepsilon_{e}$.}} We evaluate the impact of $\varepsilon_{c}$ and $\varepsilon_{e}$ on classification, where $\varepsilon _{c}$ and $\varepsilon _{e}$ determine the timing of the adaptive adjustment of repulsion (see Section ~\ref{section:4.5.2}). Figs. \ref{fig_hp} (b) report the results on CC-Cora with both $\varepsilon_{c}$ and $\varepsilon_{e}$ being varied from 0 to 2. 
When $\varepsilon_{c}$ and $\varepsilon_{e}$ are set to 0, the model pushes similar nodes apart, which can result in lower classification performance. When $\varepsilon_{c}$ and $\varepsilon_{e}$  are both larger than 1, the training performance of CHGNN degrades as there is almost no adjusted repulsion. The best performance is achieved for  $0.6<\varepsilon _{c}<1.6$ and $0.2<\varepsilon _{e}<0.6$. This is because there are significantly fewer clusters than hyperedges that contain higher-order relationships in CC-Cora. Only when the similarity of the embeddings of the constructed hyperedges is sufficiently small,  we can ensure that the model can correctly adjust the adaptive repulsion by similarity.

\noindent\textit{\textbf{Effect of $\lambda _{h}$, $\lambda _{nc}$, $\lambda _{ne}$, and $\lambda _{ec}$.}} We evaluate the impact of $\lambda _{h}$, $\lambda _{nc}$, $\lambda _{ne}$, and $\lambda _{ec}$ on classification, where $\lambda _{h}$ controls the hyperedge homogeneity loss (see Section~\ref{section:4.4.1}) and $\lambda _{nc}$, $\lambda _{ne}$, and $\lambda _{ec}$ control the cross-validation contrastive losses (see Section~\ref{section:4.4.2}). Figs. \ref{fig_hp} (c)--(f) report the results on CC-Cora, CC-Citeseer, and CA-Cora, respectively, with $\lambda _{h}$, $\lambda _{nc}$, $\lambda _{ne}$, and $\lambda _{ec}$ being varied from 0 to 1. 
For  $\lambda _{h}$, CHGNN performs the best when $0<\lambda _{h}<0.2$. Although the loss avoids nodes aggregating too much heterogeneous information, the distributions of hyperedge homogeneity and categories are not exactly the same. 
For  $\lambda _{nc}$, CHGNN underperforms when $\lambda_{nc} = 0$. In particular, cross-validation between nodes and clusters allows nodes to learn clustering information while maintaining their own features. CHGNN performs the best when $0.1<\lambda _{nc}<0.7$. This is because an excessive weight can weaken the difference between the node and cluster embeddings and disturb the embedding distribution of the nodes or clusters.
For  $\lambda _{ne}$, we see that when $\lambda _{ne}< 0.5$, the node embedding ignores some of the higher-order information in the hyperedges, resulting in lower accuracy. When $0.5<\lambda _{ne}< 1$, CHGNN corrects the node embedding distribution using cross-validation between nodes and hyperedges, thus 
improving the performance.
For  $\lambda _{ec}$, CHGNN performs the best when $0<\lambda _{ec}<0.2$, which is attributed to the difference between the distribution of node cluster embeddings and the hyperedge embeddings caused by the hyperedge homogeneity.

\vspace{-3mm}
\subsection{Ablation Study}
\label{section:5.5}

{\color{black}Table~\ref{table_ablation}  shows the results. RandAug generates views by randomly removing 20\% of the hyperedges~\cite{lee2022m}. ViewGen is the proposed adaptive hypergraph view generator embedded in CHGNN. HyperGNN is a common hypergraph encoder without  homogeneity encoding{~\cite{feng2019hypergraph}}. H-HyperGNN is the hyperedge homogeneity-aware HyperGNN encoding model.

The three losses, $\mathcal{L}_{h}$, $\mathcal{L}_{cl}$, and $\mathcal{L}_{crocl}$, the hypergraph encoder H-HyperGNN, and the view generator ViewGen are key to the performance of CHGNN. First, the proposed $\mathcal{L}_{h}$ is not suitable for HyperGNN, as the accuracy of $V_B$ drops by 0.2\%--2\% over $V_A$. Specifically, HyperGNN does not consider homogeneity during aggregation, making it impossible to fit the distribution of the hyperedge homogeneity. In contrast, $\mathcal{L}_{h}$ is suitable for H-HyperGNN. Thus, the accuracy of $V_D$ is increased by 0.1\%--0.9\% over  $V_C$. This is because H-HyperGNN adjusts aggregation weights based on hyperedge homogeneity.

Next, ViewGen outperforms RandAug, as indicated by $V_k\,(k= E, \cdots, J)$ consistently achieving better results than $V_j\,(j= K, \cdots, P)$.  The reason is that the view generator adaptively preserves important hyperedges and nodes, while random augmentation fails to ensure feature consistency between the view and the original hypergraph. 

Finally, when generating views by RandAug, adopting $\mathcal{L}_{crocl}$ yields better performance than when adopting the joint contrastive loss (i.e., $\mathcal{L}_{cl}$ + $\mathcal{L}_{crocl}$). Thus, the accuracy of $V_F$ increases by  0.1\%--1.5\% over $V_G$, and the accuracy of $V_I$ increases by  0.3\%--1.6\% over $V_J$. This is because random sampling misses critical information, which hurts the node clustering performance of nodes. Thus, nodes from different classes are grouped into a cluster, which reduces the training performance of the cluster-level contrastive loss. The same applies to the hyperedge-level contrastive loss. However, the cross-validation contrastive loss is capable of working properly because it compares the same nodes with the clusters or hyperedges they belong to.
}

\vspace{-3mm}
\section{Conclusion}\label{sec:conclusion}

We present CHGNN, a  hypergraph neural network model that utilizes CL to learn from labeled and unlabeled data. CHGNN generates hypergraph views using adaptive hypergraph view generators that assign  augmentation operations to each hyperedge and performs hypergraph convolution to capture both structural and attribute information of nodes.
CHGNN takes advantage of both labeled and unlabeled data during training by incorporating a contrastive loss and a semi-supervised loss. Moreover, CHGNN enhances the training by adaptively adjusting temperature parameters of the contrastive loss functions. This enables learning  of node embeddings that better reflect their class memberships.  Experiments on nine real-world datasets confirm the effectiveness of CHGNN---it outperforms nineteen state-of-the-art semi-supervised learning and
contrastive learning models in terms of classification accuracy. Overall, CHGNN's performance is particularly robust in scenarios with limited labeled data, making it applicable in many real-world settings. In future research, it is of interest to explore CHGNN on other downstream tasks, e.g., hyperedge prediction, and to extend CHGNN to heterogeneous hypergraphs.

\vspace{-2mm}
\section*{Acknowledgment}
This work is supported by the National Natural Science Foundation of China (62072083, U23B2019) and the Fundamental Research Funds of the Central Universities (N2216017).  

\vspace{-2mm}
\bibliographystyle{IEEEtran}
\bibliography{CHGNN}

\vspace{-15mm}
\begin{IEEEbiography}
[{\includegraphics[width=1in,height=1.25in,clip,keepaspectratio]{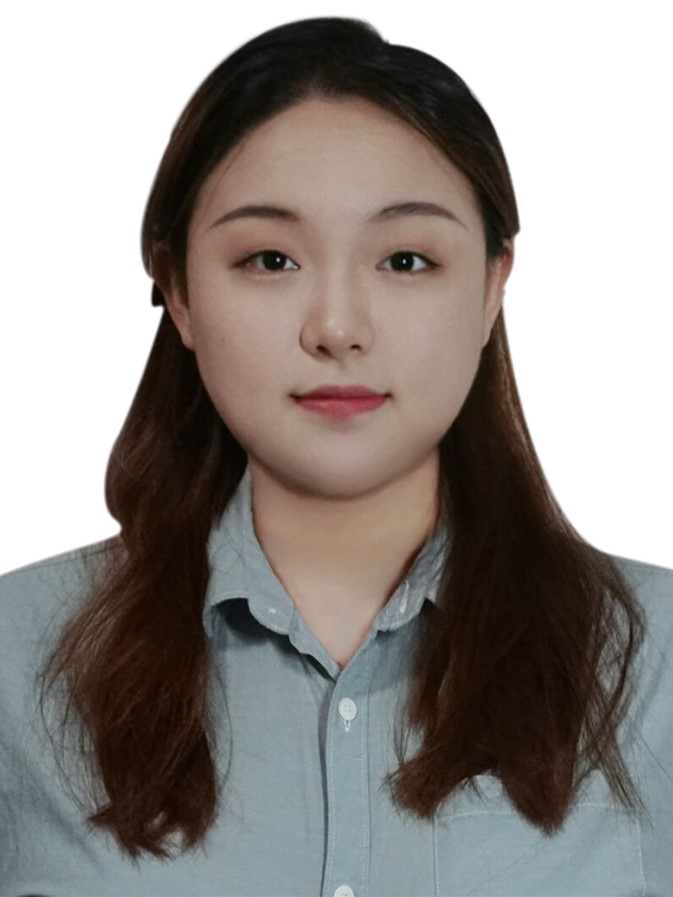}}]
{Yumeng Song} received the M.S. degree in computer software and theory from Northeastern University,
China, in 2019.  She is currently working toward her Ph.D. degree in the School of Computer Science and Engineering of Northeastern University. Her current
research interests include graph neural networks.
\end{IEEEbiography}

\vspace{-0.4cm}
\begin{IEEEbiography}
[{\includegraphics[width=1in,height=1.25in,clip,keepaspectratio]{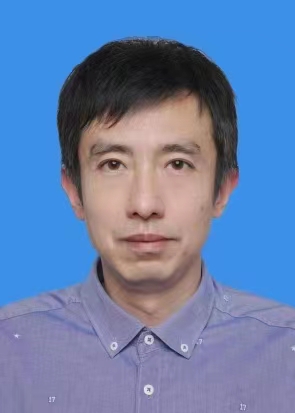}}]
{Yu Gu}
received the Ph.D. degree in computer software and theory from Northeastern University,
China, in 2010. He is currently a professor at
Northeastern University, China. His current
research interests include big data processing,
spatial data management, and graph data management. He is a senior member of China Computer Federation (CCF).
\end{IEEEbiography}

\vspace{-0.4cm}
\begin{IEEEbiography}
[{\includegraphics[width=1in,height=1.25in,clip,keepaspectratio]{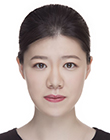}}]
{Tianyi Li} received the Ph.D. degree from Aalborg University, Denmark, in 2022. She is currently an assistant professor with the Department of Computer Science, Aalborg University. Her research concerns spatio-temporal data management and analytics, knowledge integration, and graph neural network. She is a member of IEEE.
\end{IEEEbiography}

\vspace{-0.4cm}
\begin{IEEEbiography}
[{\includegraphics[width=1in,height=1.25in,clip,keepaspectratio]{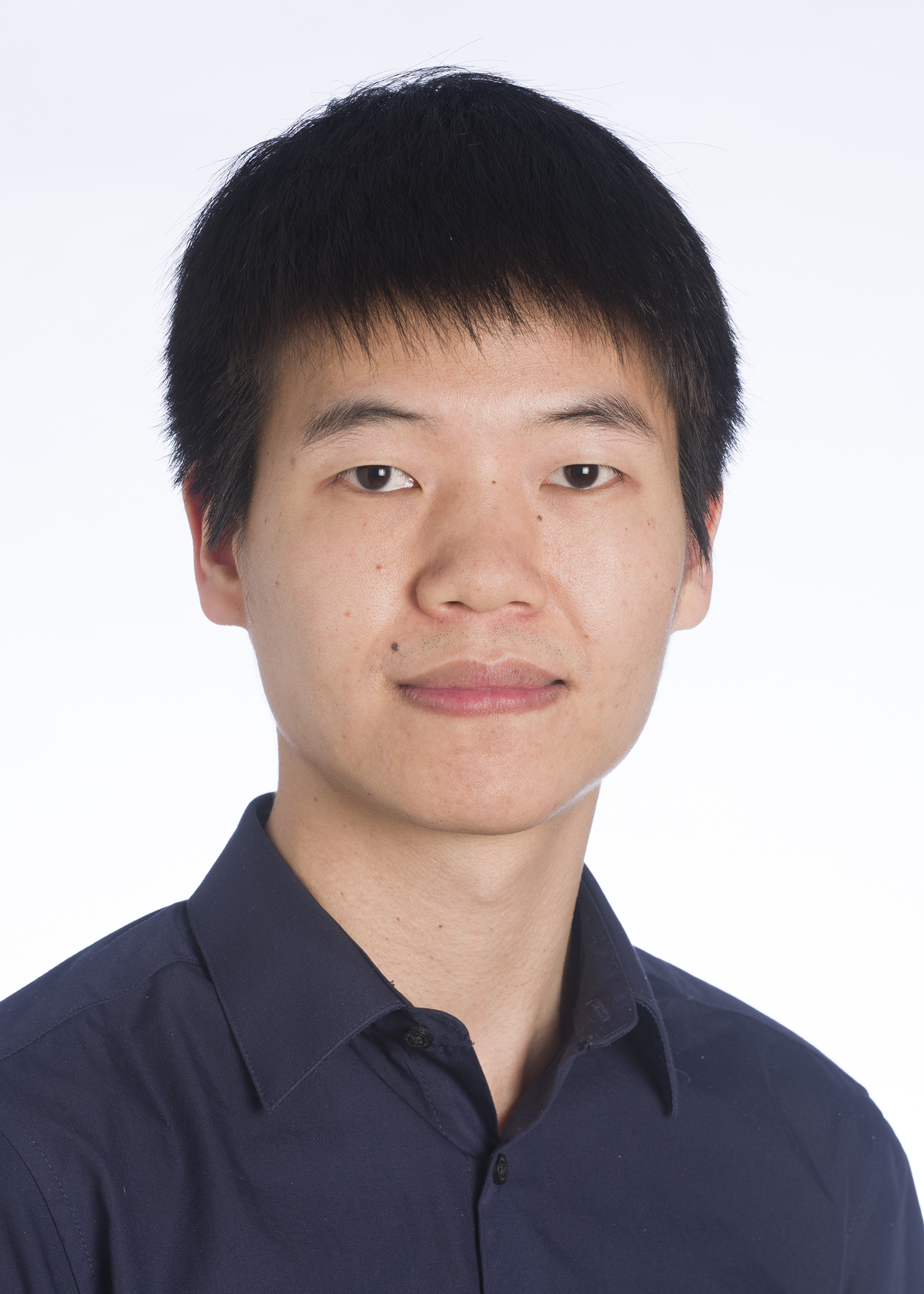}}]
{Jianzhong Qi}
received the Ph.D. degree from the University of Melbourne, in 2014. He is a lecturer with the School of Computing and Information Systems, University of Melbourne. He has been
an intern with the Toshiba China R\&D Center
and Microsoft, Redmond, WA, in 2009 and 2013,
respectively. His research interests include spatio-temporal databases, location-based social networks, and information extraction.
\end{IEEEbiography}

\vspace{-0.4cm}
\begin{IEEEbiography}
[{\includegraphics[width=1in,height=1.25in,clip,keepaspectratio]{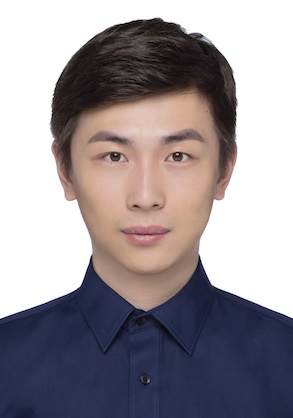}}]
{Zhenghao Liu}
received the Ph.D. degree in computer software and theory from Tsinghua University, 
China, in 2021. He is currently an associate professor at
Northeastern University, China. His current
research interests include artificial intelligence, information retrieval in natural language processing, and automated question and answer.
\end{IEEEbiography}

\vspace{-0.4cm}
\begin{IEEEbiography}
[{\includegraphics[width=1in,height=1.25in,clip,keepaspectratio]{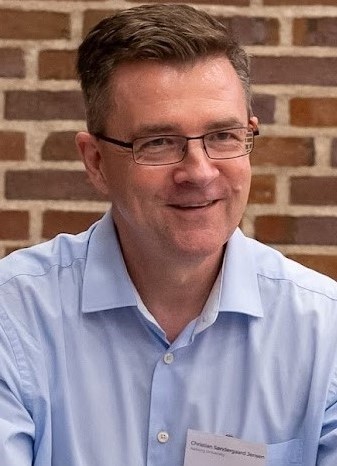}}]{Christian S. Jensen}
received the Ph.D. degree from Aalborg University, Denmark, where he is  a professor. His research concerns data analytics and management with a focus on temporal and spatiotemporal data. He is a fellow of the ACM and IEEE, and he is a member of the Academia Europaea, the Royal Danish Academy of Sciences and Letters, and the Danish Academy of Technical Sciences.
\end{IEEEbiography}

\vspace{-0.4cm}
\begin{IEEEbiography}
[{\includegraphics[width=1in,height=1.25in,clip,keepaspectratio]{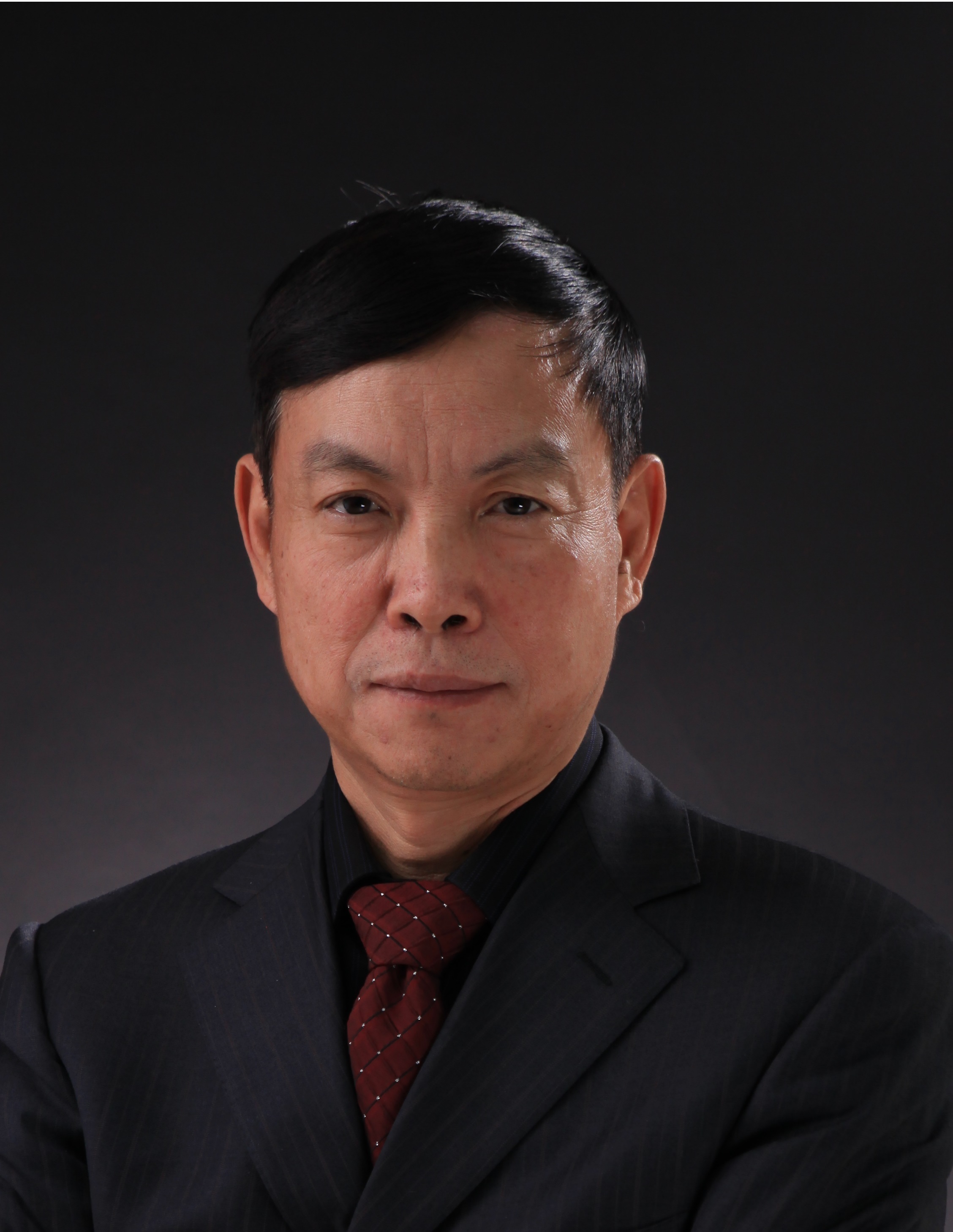}}]{Ge Yu}
received the Ph.D. degree in computer science from the Kyushu University of Japan, in 1996. He is currently a professor at the Northeastern University of China. His research
interests include distributed and parallel database, data integration, and 
graph data management.  He is a fellow of CCF and a member of
the IEEE and ACM.
\end{IEEEbiography}

%

\end{document}